\documentclass[journal]{new-aiaa}
\usepackage[utf8]{inputenc}

\usepackage{graphicx}
\usepackage[version=4]{mhchem}
\usepackage{siunitx}
\usepackage{longtable,tabularx}
\usepackage[T1]{fontenc}
\usepackage{xfrac} % assumes amsmath package installed
\usepackage{amsmath}

\usepackage{amsthm}
\usepackage{bm}
\usepackage{epsfig} % for postscript graphics files
\usepackage{algorithmicx}
\usepackage{algorithm}
\usepackage[noend]{algpseudocode}
\usepackage{subcaption}
\usepackage{setspace}

\theoremstyle{definition}
\newtheorem{definition}{Definition}[section]
\newtheorem{problem}{Problem}
\newtheorem{proposition}{Proposition}[section]
\newcommand{\norm}[1]{\left\lVert#1\right\rVert}
\DeclareMathOperator{\atantwo}{atan2}
\DeclareMathOperator{\sign}{sgn}
\newcommand{\Mod}[1]{\ (\mathrm{mod}\ #1)}

\title{Collision Avoidance for Unmanned Aerial Vehicles in the Presence of Static and Moving Obstacles \footnote{This is a revised version of the original published in the AIAA Journal of Guidance, Controls and Dynamics, Vol. 43, Iss. 1}}

\author{Andrei Marchidan \footnote{PhD Candidate,  Department of Aerospace Engineering and Engineering Mechanics, andrei.marchidan@utexas.edu} and Efstathios Bakolas \footnote{Associate Professor, Department of Aerospace Engineering and Engineering Mechanics, AIAA Senior Member, bakolas@austin.utexas.edu}}
\affil{The University of Texas at Austin, Austin, TX, 78712}

\begin{document}
	
	\maketitle
	
	\begin{abstract}
		This paper presents a new collision avoidance procedure for unmanned aerial vehicles in the presence of static and moving obstacles. The proposed procedure is based on a new form of local parametrized guidance vector fields, called collision avoidance vector fields, that produce smooth and intuitive maneuvers around obstacles. The maneuvers follow nominal collision-free paths which we refer to as streamlines of the collision avoidance vector fields. In the case of multiple obstacles, the proposed procedure determines a mixed vector field that blends the collision avoidance vector field of each obstacle and assumes its form whenever a pre-defined distance threshold is reached. Then, in accordance to the computed guidance vector fields, different collision avoidance controllers that generate collision-free maneuvers are developed. Furthermore, it is shown that any tracking controller with convergence guarantees can be used with the avoidance controllers to track the streamlines of the collision avoidance vector fields. Finally, numerical simulations demonstrate the efficacy of the proposed approach and its ability to avoid collisions with static and moving pop-up threats in three different practical scenarios. 
	\end{abstract}
	
	\section*{Nomenclature}
	
	{\renewcommand\arraystretch{1.0}
		\noindent\begin{longtable*}{@{}l @{\quad=\quad} l@{}}
			UAV & unmanned aerial vehicle \\
			$\mathbb{R}^n$ & set of $n$-dimensional real vectors\\
			$\mathbb{Z}$ & set of integers\\
			CAVF & collision avoidance vector field\\
			$\emptyset$ & empty set\\
			$\langle \cdot, \cdot \rangle$ & dot product\\
    		$\angle(\cdot,\cdot)$ & angle between two vectors\\
    		$V$ & UAV speed\\
    		$x$ & UAV position on the $x$-axis of the inertial frame \\
    		$y$ & UAV position on the $y$-axis of the inertial frame \\
    		$\bm{p}$ & UAV position vector\\
    		$\psi$ & UAV heading\\
    		$\psi_d$ & UAV desired heading\\
    		$u(t)$ & UAV steering control input at time $t$\\
    		$\xi_{[t_i,T]}$ & UAV position trajectory at time interval $[t_i,T]$\\
    		$T$ & free final time\\
			$r_o$ & obstacle radius\\
    		$r_s$ & sensing range\\
			$\mathcal{S}(t)$ & sensing space at time $t$\\
			$\mathcal{J}(t)$ & index set for all registered obstacles at time $t$\\
			$\mathcal{O}(t)$ & set of all inadmissible UAV positions at time $t$\\
			$\mathcal{F}(t)$ & set of all admissible UAV positions at time $t$\\
			$(\cdot)^\mathsf{C}$ & complement of a set operator\\
			$h(\bm{p})$ & collision avoidance vector field\\
			$\bm{p_o}$ & position of registered obstacle\\
    		$x_o$ & obstacle position on the $x$-axis of the inertial frame \\
			$y_o$ & obstacle position on the $y$-axis of the inertial frame \\
			$h_s(\bm{p})$ & collision avoidance vector field for a static obstacle\\
			$\bm{e_x}$ & $x$-axis unit vector\\
			$\bm{e_y}$ & $y$-axis unit vector\\
			$\bm{e_r}$ & unit vector for the line-of-sight to the obstacle\\
			$\bm{e_\theta}$ & unit vector for the obstacle tangent, perpendicular to $\bm{e_r}$\\
			${^i(\cdot)}$ & inertial frame superscript\\
			${^b(\cdot)}$ & obstacle moving frame superscript\\
			$r$ & radial distance of an agent from an obstacle's center\\
			$\theta$ & angle between the $x$-axis and the line-of-sight between $\bm{p}$ and $\bm{p_o}$\\
			$\beta$ & angle between the desired trajectory direction and obstacle line-of-sight\\
			% this is where the transition function is defined
			$V_o$ & obstacle speed\\
			$\theta_o$ & obstacle direction with respect to $x$-axis\\
			$V_b$ & UAV speed in the moving obstacle frame\\
			$\psi_b$ & UAV desired heading in the moving obstacle frame\\
			$h_d(\bm{p})$ & collision avoidance vector field for a dynamic obstacle\\
			$\phi$ & angle between the line-of-sight and the UAV velocity\\
			$u_s(t)$ & collision avoidance control input for a static obstacle\\
			$u_d(t)$ & collision avoidance control input for a dynamic obstacle\\
			$u_m(t)$ & collision avoidance control input for multiple obstacles\\
			$u_t(t)$ & tracking control input for multiple obstacles CAVFs\\ 
	\end{longtable*}}

	\section{Introduction}
	Unmanned Aerial Vehicles (UAVs) have been widely used by both military and civil entities in missions ranging from surveillance, to search-and-rescue, to convoy protection, to pay-load delivery for rescue situations or even for commercial businesses. In accomplishing these tasks, the autonomous vehicle is required to navigate without supervision in environments populated with both static and moving obstacles. Specifically, while it is following planned trajectories that are in accordance with high-level specifications, such as, flying through different waypoints or maintaining a specific course given by a path planning protocol, the UAV must be able to perform maneuvers to avoid pop-up threats or obstacles that may not have been considered for the original trajectory plan. Thus, the need for decentralized, reactive, computationally inexpensive and fast collision avoidance modules arises. 
	
%	Motivation for the system used - like frew said... 
	Flight control systems for unmanned aerial vehicles are well-studied and mature methodologies provide feedback controllers that guarantee accurate tracking of reference pose and orientation. As such, it is a common assumption in the path planning community for UAVs to assume constant altitude operation and to approximate the system with a planar kinematic Dubins model \cite{frew2008coordinated}. This implies that the planning search space reduces to two geometric dimensions and that navigation depends only on speed and steering or heading control. With these assumptions, the collision avoidance problem can be tackled in different ways, depending on additional mission requirements. 
	
%	Then techniques used to solve the collision avoidance problem
	Some of the more notable early works tackled the collision avoidance problem by creating a graph in the autonomous agent's free configuration space, considering only the geometric requirement of finding an obstacle-free path between two points. Then, obstacle-free paths are found by applying different search algorithms that attempt to connect graph nodes while, at the same time, may optimize different metrics. Some of the most used algorithms are Dijkstra's \cite{dijkstra}, A* \cite{A*, yangZhao}, D* or D*-lite algorithms \cite{D*}. The configuration space graphs are usually created using either basic sampling techniques that depend on feature density or more complex techniques that employ more advanced space discretizations: decomposition into cells \cite{zhuLatombe}, Voronoi diagrams \cite{sharir}, projections \cite{sharir2} or retractions \cite{retract}. The graph-based search algorithms, however, rely on search space granularity for speed and accuracy and, therefore, are not very suitable for real-time applications. To speed up the process of sampling and searching for paths, a sampling-based technique that generates obstacle-free paths very fast was developed by Lavalle \cite{lavalle} and was later extended to more complex dynamics \cite{rrt*, frazzoli-agile, lqrtrees}. These new approaches, relying on randomly sampling the free configuration space, sped up the search algorithms, however, they required both re-planning in the presence of moving obstacles and a large number of samples and collision checks for very cluttered environments, still making them computationally costly.

	Other popular approaches for collision avoidance rely on curve parametrization for trajectory generation \cite{upda-ratnoo, mattei, ikeuchi}. These geometric techniques integrate dynamic and path constraints by performing waypoint parameter optimization in order to define splines, polynomial, logistic or B\'ezier curves, or even clothoids, that are feasible for more complex systems to follow. Such constraints may relate to speed, path curvature, path length, obstacle avoidance, and many other mission requirements. Solving these problems, however, requires high computational resources due to their nonlinear nature which would lead to the appearance of local minima or configurations where the autonomous agent may get stuck. Moreover, none of these parameter optimization techniques are able to account for moving obstacles without further assumptions, simplifications or re-planning.

	A similar class of algorithms used for collision avoidance are optimization-based algorithms that aim to solve nonlinear programs with different NLP (nonlinear programming) solvers \cite{sunliudai, frazzoliSDP} by applying direct or indirect numerical methods. Indirect methods require deriving the necessary optimality conditions and finding the correct adjoint variables to satisfy these conditions, while direct methods use a discretization of the state and input variables in order to reduce the optimal control problem into an NLP problem. Both of these nonlinear programming techniques cannot guarantee convergence to a feasible solution and demand good initial guesses for their decision variables, which still makes them too computationally intensive for real-time applications. 
	
	To avoid the computational complexity of the mentioned techniques, a different approach that generates obstacle-free paths by using artificial potential fields was developed in \cite{khatib}. In this method, obstacles are associated with repelling forces and target destinations are associated with an attractive force. Then, by considering these to be the only forces acting on the autonomous agent's body, the motion plan is generated through gradient descent on the artificial potential field. The method's simplicity comes with two drawbacks: the existence of local minima and the lack of a mechanism to handle input constraints, which may lead to infeasible commands that the agent must follow. These problems were ammended by the introduction of harmonic potential fields, which generate motion plans that mimic fluid flow around static obstacles \cite{waydo,libui,zikang,zadehDSA}. This technique, however, doesn't come with any guarantees that the generated control inputs will be feasible. As such, in an attempt to include kinematic constraints, a new fluid motion planner was introduced in \cite{dLau} and \cite{owen} for curvature and speed constraints. The new method, however, cannot handle moving obstacles.
		 
	In this paper, a new approach that uses a parameterized vector field is proposed for generating constant speed collision avoidance maneuvers around static and moving obstacles. The idea of a parametrized vector field has been used in past works to generate guidance laws for tracking different motion patterns \cite{lawrence2008lyapunov, frew2008coordinated}. The tracking vector fields were obtained using Lyapunov functions that would guarantee convergence to the desired motion plans. By contrast, the parametrized vector fields proposed herein are generated using a new approach that depends on obstacle proximity to modulate the agent's velocity such that it becomes tangential to the obstacle surface before colliding with it. The parameters can be used to adjust the agent's behavior around the obstacle, by performing the transition between obstacle-free motion and collision avoidance from different distances and with different intensities. A steering controller is described for the static and moving obstacle case, respectively, and a control algorithm is proposed for the multiple obstacle case. As such, depending on the obstacles' velocity and their proximity, the UAV may have to use one of the proposed controllers or the proposed control algorithm for avoiding collisions with constant speed. Moreover, the UAV may be required at times to switch between controllers due to the appearance of new obstacles along its path. As a consequence, collision avoidance may not be guaranteed since the switch may not lead to motion continuity between obstacle-free motion and collision avoidance or, in other words, may bring new initial conditions that would not correspond to the desired collision avoidance vector field. Therefore, the use of a tracking controller with proven convergence guarantees is proposed. Tracking the desired inputs leads to the desired motion plan, provided by the collision avoidance vector fields, and obstacle avoidance is guaranteed. 
	
	This paper is organized as follows. In Section \ref{sec::Problem}, the collision avoidance problem is defined. In Section \ref{sec::CAVF}, a new formulation of a guidance vector field that accomplishes collision avoidance is provided. In Section \ref{sec::Control} these vector fields are combined appropriately in order to yield a feasible controller that satisfies the collision avoidance problem requirements. Numerical simulations are illustrated and discussed in Section \ref{sec::sims}. Finally, in Section \ref{sec::concl} concluding remarks are provided.

	\section{Formulation of the Collision Avoidance Problem}
	\label{sec::Problem}
	In this section, the system state space model and the mathematical formulation of the collision avoidance problem are provided. To this end, consider an UAV flying in a horizontal plane with constant forward speed $V$ and direction $\psi$, relative to the $x$-axis of the inertial frame of motion which is fixed to the plane, as illustrated in Figure \ref{fig::coordinates}. The UAV's full state is expressed by its position $\bm{p} := [x,y]^T \in \mathbb{R}^2$ and orientation $\psi \in \mathbb{R}$. The equations of motion are given by the following kinematic model:
	\begin{equation}
	\begin{aligned}
	\dot{x}(t) &= V \cos\psi(t), \quad &x(t_i) &= x_i, \\
	\dot{y}(t) &= V \sin\psi(t), \quad &y(t_i) &= y_i,\\
	\dot{\psi}(t) &= u(t), \quad &\psi(t_i) &= \psi_i, 
	\end{aligned}
	\label{eqn::SysInertial}
	\end{equation}
	where $u$ is the steering rate input, and $[x_i,y_i]^T \in \mathbb{R}^2$ and $\psi_i \in \mathbb{R}$ are the vehicle position and heading, respectively, at the initial time $t_i$. 
	The UAV position trajectory obtained by integrating system \eqref{eqn::SysInertial} forward in time is denoted by $\xi_{[t_i,T]}:[t_i,T] \to \mathbb{R}^2$, where $T > t_i$ is the free final time and $\xi_{[t_i,T]} := \{ \bm{p}(t) : t \in [t_i,T] \}$. 

	\begin{figure}
		\centering
		\includegraphics[width=3.25in]{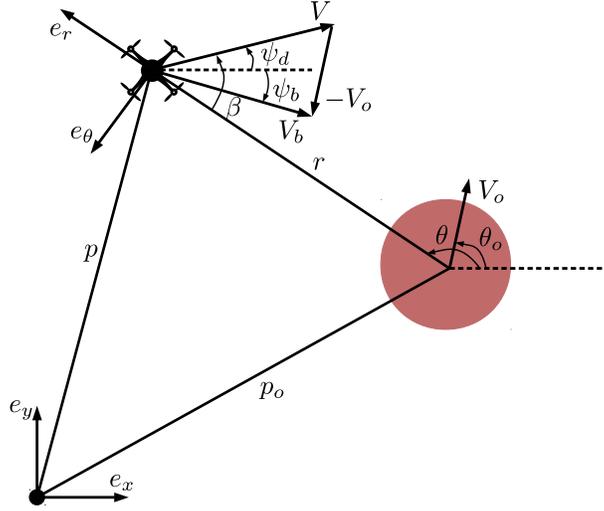}
		\caption{Diagram illustrating the coordinate systems and the state variables used in the collision avoidance problem.}
		\label{fig::coordinates}
	\end{figure}
	
	Further, it is assumed that the UAV is initially traveling at constant speed according to a high-level planning module which provides a constant desired steering angle, denoted by $\psi_d \in \mathbb{R}$. This only suggests that, for the planning period, the UAV's goal is to maintain the same heading and that any deviation from a predefined path would be taken care of by the high-level planner, which is outside the scope of this paper. 
	
	The UAV is moving in an environment that may be populated with static and moving obstacles which can represent no-fly zones, other UAVs, airplanes or physical obstacles. Thus, for planning purposes, it is assumed that each obstacle can be bounded by a circular region of a specific radius $r_o > 0$, determined by the largest obstacle dimension. Moreover, it is assumed that the UAV has a limited sensing range modeled as a circular region of radius $r_s$, determined by its sensor capabilities, centered at the UAV's geometric center. As such, the UAV sensing region at time $t$ may be modeled in accordance with its sensing range:
	\begin{align}
	\mathcal{S}(t) = \{ [\mathsf{x},\mathsf{y}]^T \in \mathbb{R}^2 : \sqrt{(x(t)-\mathsf{x})^2 + (y(t)-\mathsf{y})^2} \leq r_s \},
	\label{def::PlanningSet}
	\end{align}
	Due to the sensor limitations, the obstacles can be viewed as ``pop-up'' motion constraints since they are registered by the UAV only when they enter the circular sensing region. Let $\mathcal{J}(t) = \{1,2,\dots,n(t)\}$ be the index set for all registered obstacles at time $t$, where $n(t)$ is the total number of obstacles found. Then, the set of all inadmissible UAV positions is defined by:
	\begin{align}
	\mathcal{O}(t) = \{ [\mathsf{x},\mathsf{y}]^T\in\mathbb{R}^2 : \sqrt{(x_o^j(t)-\mathsf{x})^2 + (y_o^j(t)-\mathsf{y})^2} \leq r_o^j, ~\forall j \in \mathcal{J}(t) \},
	\label{def::InadmissibleSet}
	\end{align} 
	where $[x_o^j(t),y_o^j(t)]^T \in \mathcal{S}(t)$ is the position of obstacle $j$ inside the sensing region and $r_o^j$ is the obstacle radius. Therefore, the UAV motion is constrained to the free space inside the sensing region $\mathcal{S}(t)$, which is defined as the set of all admissible UAV positions:
	\begin{align}
	\mathcal{F}(t) = \mathcal{S}(t) \cap \mathcal{O}(t)^\mathsf{C},
	\label{def::AdmissibleSet}
	\end{align}
	The sensing region and the admissible and inadmissible sets are illustrated in Figure \ref{fig::Spaces}. 
	
	\begin{figure}
		\centering
		\includegraphics[width=3.25in]{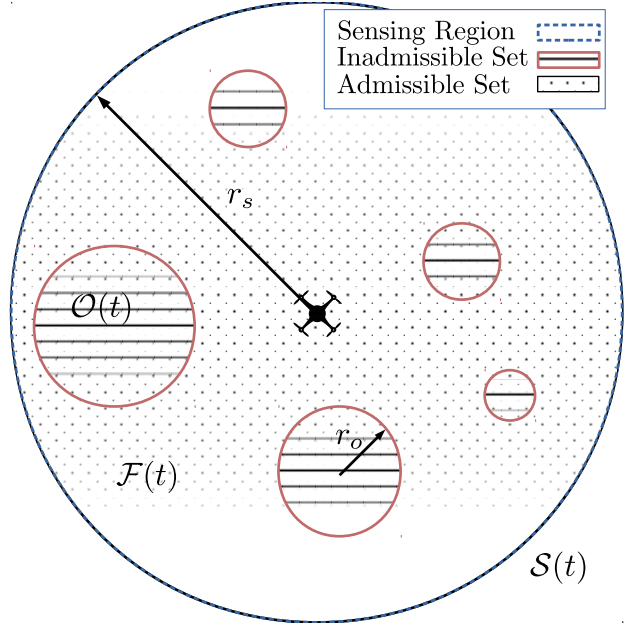}
		\caption{Diagram illustrating the spaces involved in the collision avoidance problem.}
		\label{fig::Spaces}
	\end{figure}
	
	The collision avoidance problem can be stated as follows:
	\begin{problem}
		Consider an UAV whose equations of motion are given in \eqref{eqn::SysInertial}, with a sensing radius $r_s$, moving through an environment with static and moving obstacles, such that $\mathcal{O}(t) \neq \emptyset$. Find the steering rate input $u(t)$ such that the trajectory generated by \eqref{eqn::SysInertial}, $\xi_{[t_i,T]}$, avoids collisions with any registered obstacles, while, at the same time, maintains the same course that it had before initiating the collision avoidance procedure:
		\begin{align*}
		\xi_{[t_i,T]} \in \mathcal{F}(t), ~\forall t \in [t_i,T] \text{ and } \psi(T) = \psi_d.
		\end{align*}
	\end{problem}	
	To solve this problem, a novel approach that determines inputs through a special class of vector fields that ensure collision avoidance and their corresponding controllers is proposed in the following sections. Note that the approach described in Section \ref{sec::CAVF} can be used for different system models and as such will be presented as an application for any autonomous agent. 

	\section{Collision Avoidance Vector Fields}
	\label{sec::CAVF}
	Consider an autonomous agent moving according to system \eqref{eqn::SysInertial} around an obstacle of radius $r_o$. Since collision avoidance requires a study of the relative motion between the agent and the obstacle, two coordinate frames are considered: one inertial frame, fixed to a point in space and determined by the unit vectors $\bm{e_x}$ and $\bm{e_y}$; and one moving frame, fixed at the obstacle geometric center and determined by unit vectors $\bm{e_r}$ and $\bm{e_\theta}$, as illustrated in Figure \ref{fig::coordinates}.
	
	To avoid collision with an obstacle, the agent's trajectory must not penetrate its surface at any point in time. This can be accomplished by ensuring that the agent's velocity along the line-of-sight to the obstacle is non-negative at the obstacle surface. A guidance law that achieves this behavior can be developed using a vector field with radial velocities that are non-negative at the obstacle surface. As such, the concept of a collision avoidance vector field is introduced.
	\begin{definition}{}
		Given a spherical obstacle of radius $r_o$ that is centered at $\bm{p_o} \in \mathbb{R}^2$, its collision avoidance vector field is defined as a spatially dependent vector field $h(\cdot) : \mathbb{R}^2 \to \mathbb{R}^2$ with the property that $\exists ~\alpha \geq 0$ such that 
		\begin{equation}
		\langle h(\bm{p_r}), \bm{e_r} \rangle \geq \alpha, ~\forall \bm{p_r} \in \mathbb{R}^2 \text{ that satisfies } \norm{\bm{p_r}-\bm{p_o}} = r_o.
		\label{eqn::cavf_def}
		\end{equation} 
		\label{def::CAVF}
	\end{definition}
	Any agent whose equations of motions determine a vector field that matches the collision avoidance vector field given in Definition \ref{def::CAVF} will move away from the obstacle surface in an attempt to avoid collisions with the particular obstacle.
	
	\subsection{Local Collision Avoidance Vector Field for a Single Obstacle}	
	In this section, an approach for generating parametrized collision avoidance vector fields around one static obstacle when the agent is supposed to move with constant speed, $V$, along a constant direction, determined by $\psi_d$, is proposed. Consider a static obstacle of radius $r_o$ located at $\bm{p_o} = [x_o,y_o]^T \in \mathcal{S}(t)$. Let $h_s(\bm{p}) = [\dot{r},r\dot{\theta}]^T$ be a spatially dependent vector field in the relative frame, where $\bm{p}\in \mathbb{R}^2$ is a point in space, $r = \norm{\bm{p}-\bm{p_o}}$, and $\theta \in \mathbb{R}$ is the angle that the line-of-sight between points $\bm{p}$ and $\bm{p_o}$ makes with the inertial $x$-axis of a coordinate system centered at $\bm{p_o}$, as illustrated in Figure \ref{fig::coordinates}. Since the obstacle is stationary, collision avoidance is achieved if $\exists ~\alpha \geq 0$ such that \eqref{eqn::cavf_def} is satisfied, which implies that $\dot{r} \geq 0$ at the obstacle surface $r = r_o$. Therefore, consider the following system:
	\begin{equation}
	\begin{aligned}
	\dot{r} &= -\lambda(r,\theta) V \cos \beta, \\
	\dot{\theta} &= -\sign(\sin\beta)\frac{1}{r}\sqrt{V^2-\dot{r}^2},
	\end{aligned}
	\label{eqn::CAVFstatic}
	\end{equation}
	where $V$ is the agent's speed, $\beta = \angle([\cos\psi_d,\sin\psi_d]^T,-\bm{e_r})$ is the angle between the desired trajectory direction and the line-of-sight to the obstacle, and $\lambda(\cdot,\cdot) : [r_o,\infty]\times \mathbb{R} \to [0,1]$ is a continuous function with the property that $\lambda(r_i,\theta) = 1$ and $\lambda(r_o,\theta) = 0$, $\forall \theta \in \mathbb{R}$. Let $\lambda(r,\theta)$ be defined as follows:
	\begin{align}
	\lambda(r,\theta) = \begin{cases}
	-\frac{2}{\pi} \Big(\gamma(r)(\psi_d -\theta) + \theta - \psi_d - \frac{\pi}{2}\Big), & \text{if } r_o \leq r \leq r_i, ~~(\theta-\psi_d) \Mod{2\pi} \in (0,\pi/2] \\
	\gamma(r), & \text{if } r_o \leq r \leq r_i, ~~ (\theta-\psi_d) \Mod{2\pi} \in (\pi/2,3\pi/2]\\
	\frac{2}{\pi} \Big(\gamma(r)(\psi_d -\theta) + \theta - \psi_d + \frac{\pi}{2}\Big), & \text{if } r_o \leq r \leq r_i, ~~ (\theta-\psi_d) \Mod{2\pi} \in (3\pi/2,2\pi]\\
	1, & \text{if } r_i < r, ~~ \theta \in \mathbb{R}
	\end{cases}
	\label{eqn::lambda}
	\end{align}
	where 
	\begin{align}
	\gamma(r) = \displaystyle\frac{a\Bigg(\displaystyle\frac{1}{r_o-r} - \frac{1}{r-r_i}\Bigg)} {\sqrt{1+\Bigg(2a\Bigg(\displaystyle\frac{1}{r_o-r} - \frac{1}{r-r_i}\Bigg)\Bigg)^2}} + 0.5.
	\label{eqn::gamma_r}
	\end{align}
	Parameters $a > 0$ and $r_i > r_o$ can be chosen such that the vector field smoothness and reactivity are influenced so that the desired motion patterns are achieved. In other words, the distance at which the CAVF becomes active is determined by $r_i$ and the transition between obstacle-free motion and collision avoidance is performed more abruptly or gradually depending on the choice of parameter $a$, as it can be seen in Figure \ref{fig::Varying_a}. Note further that $r_i$ defines a region of influence around the obstacle and any agent inside this region of influence is considered to perform collision avoidance maneuvers. The function $\lambda$ is used for two reasons: to generate a vector field that becomes tangential at the obstacle's surface and that circulates around the obstacle while trying to maintain the original heading $\psi_d$. As such, the provided function $\lambda$ achieves the first goal by artificially scaling the vector field as $r \to r_o$, since $\gamma(r_o) = 0$ and $\lambda(r_o,\theta) = 0$ implies that $\dot{r} = 0$ at $r = r_o$. Therefore, $\alpha = 0$ satisfies \eqref{eqn::cavf_def} and the equations of motion given in \eqref{eqn::CAVFstatic} describe a CAVF. Moreover, the second goal is achieved through a smooth transition, determined by $\lambda$, from obstacle-free motion to tangential motion which reverses direction as the obstacle is cleared. An obstacle is considered cleared when the agent velocity points away from the obstacle, $\langle \bm{V}, -\bm{e_r} \rangle < 0$, where $\bm{V}$ is the agent velocity vector. One sample vector field is illustrated in Figure \ref{fig::CAVFstatic}, where $\psi_d = 0$. 
	\begin{figure}[h!]
		\centering
		\begin{subfigure}{0.23\textwidth}
			\centering
			\includegraphics[width=1.6in]{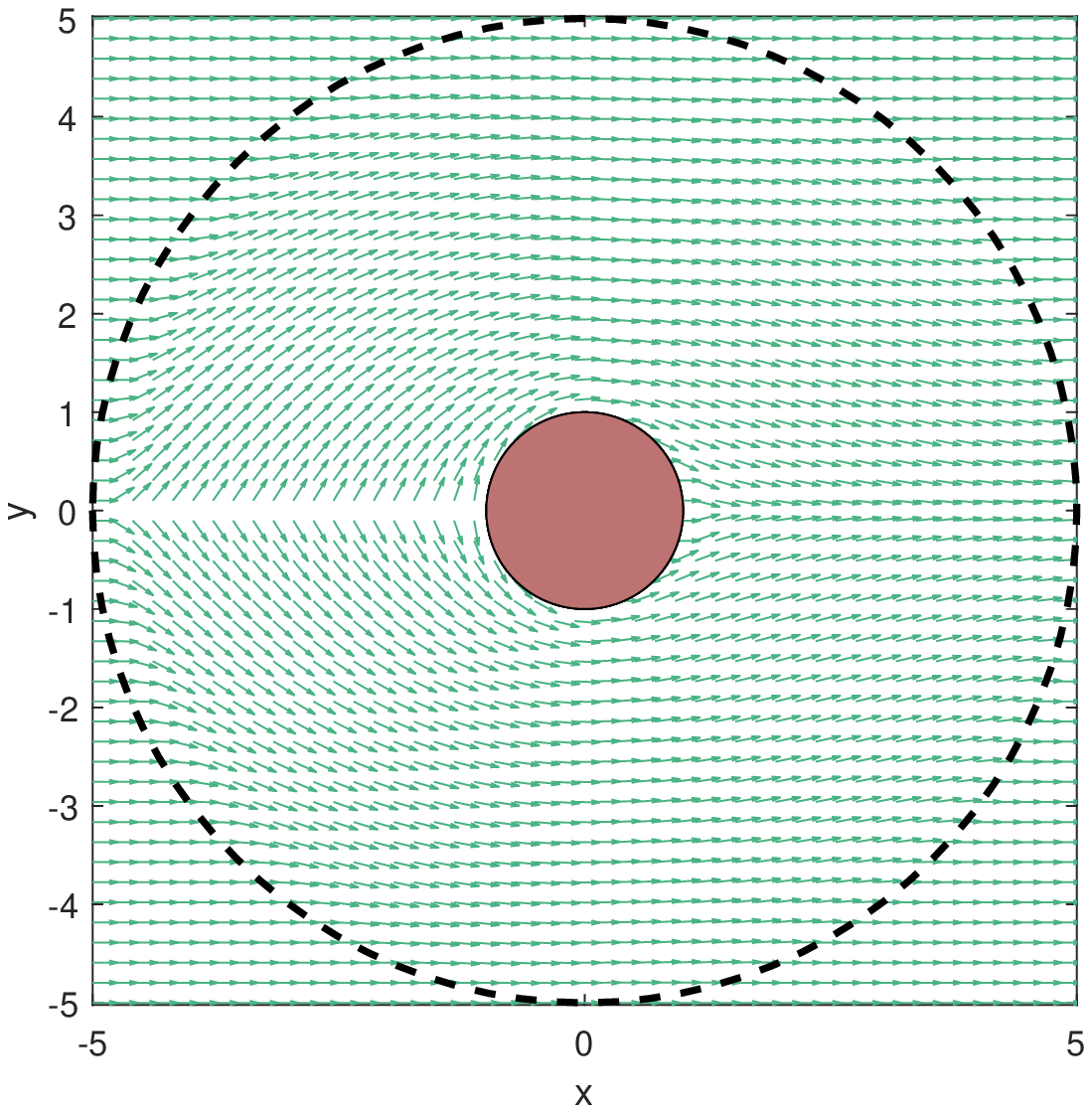}
			\caption{$a = 0.1$}
		\end{subfigure}\hspace{1 mm}
		\begin{subfigure}{0.23\textwidth}
			\centering
			\includegraphics[width=1.6in]{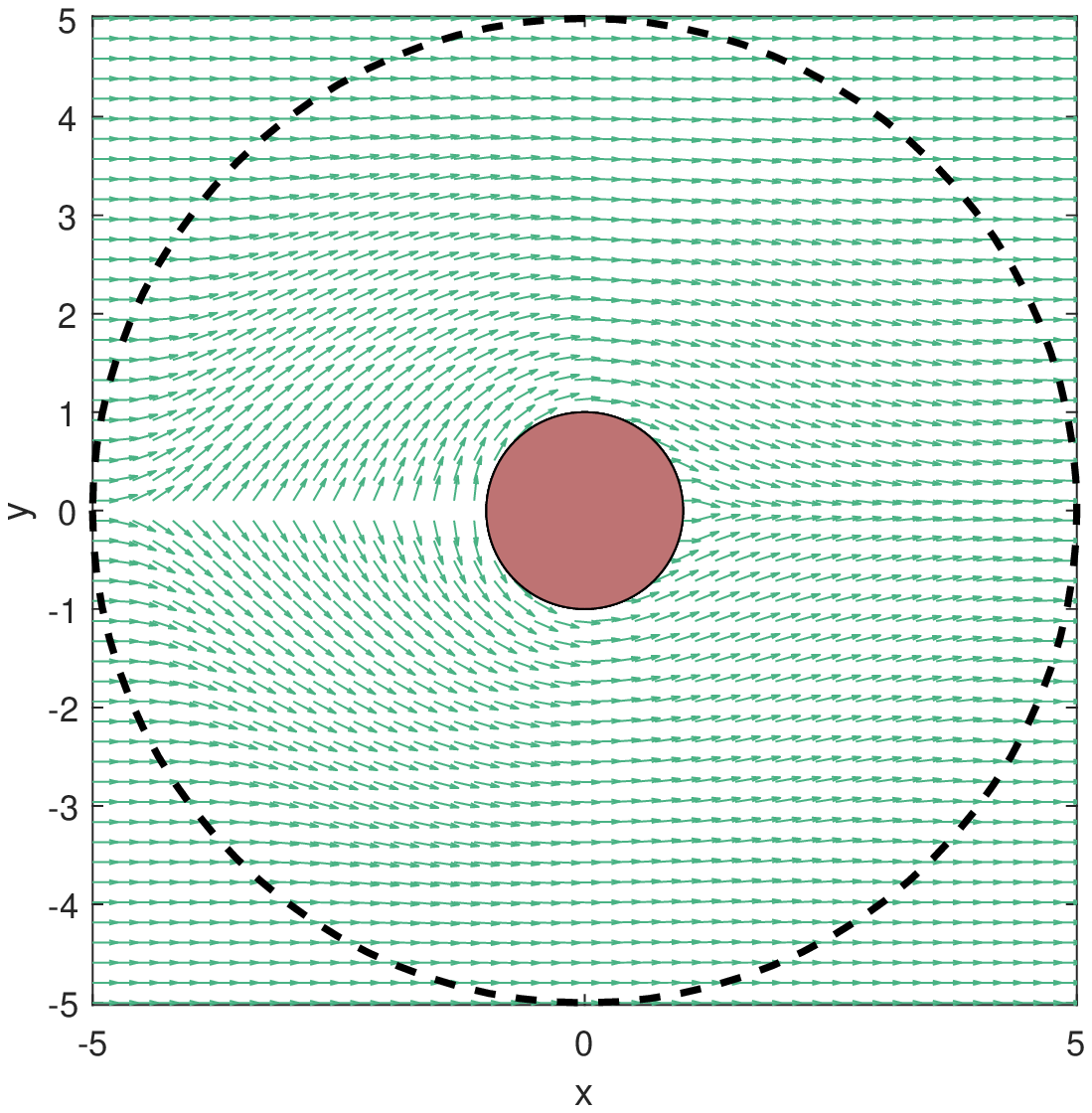}
			\caption{$a = 0.3$}
		\end{subfigure}\hspace{1 mm}
		\begin{subfigure}{0.23\textwidth}
			\centering
			\includegraphics[width=1.6in]{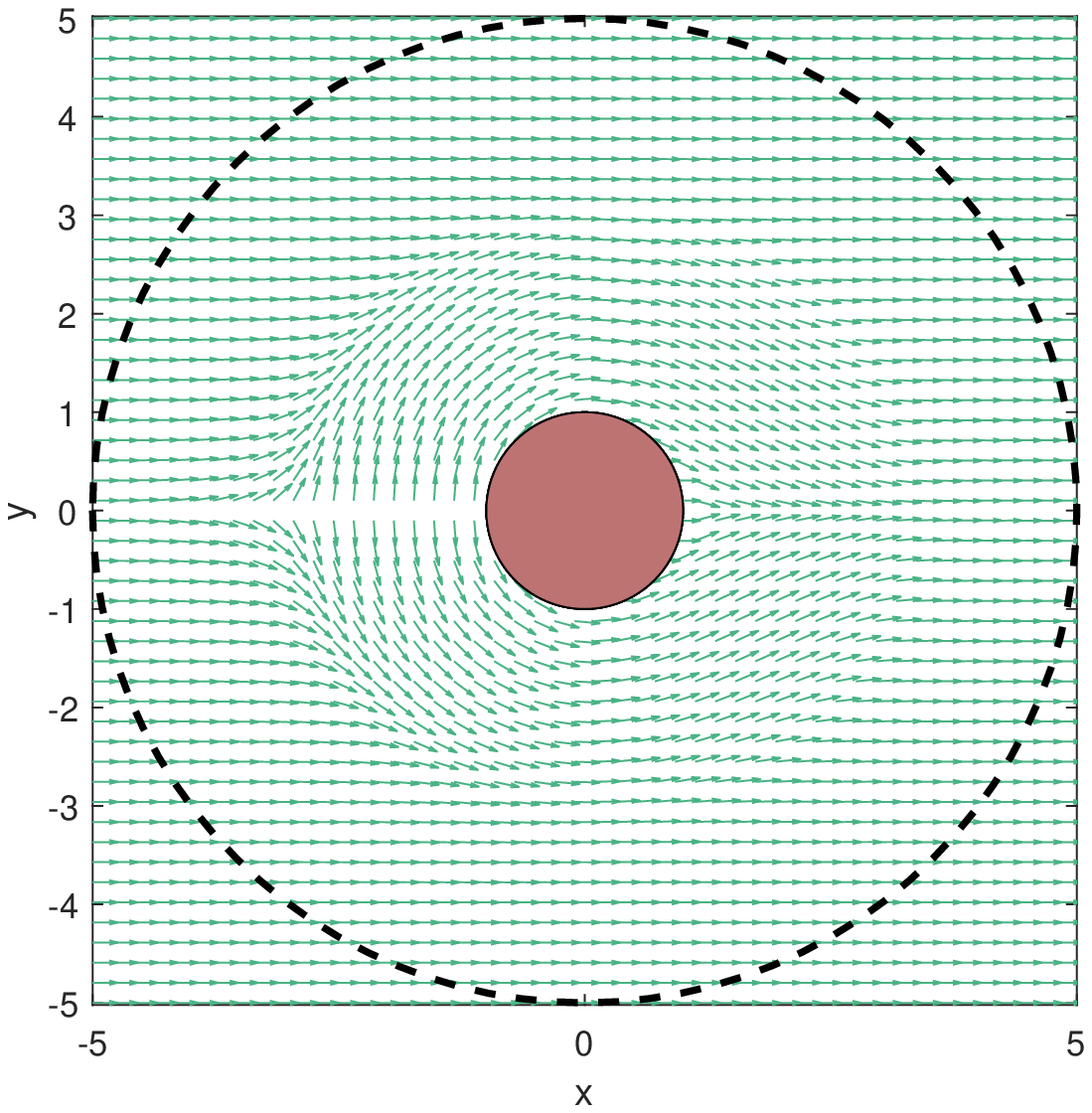}
			\caption{$a = 3$}
		\end{subfigure}\hspace{1 mm}
		\begin{subfigure}{0.23\textwidth}
			\centering
			\includegraphics[width=1.6in]{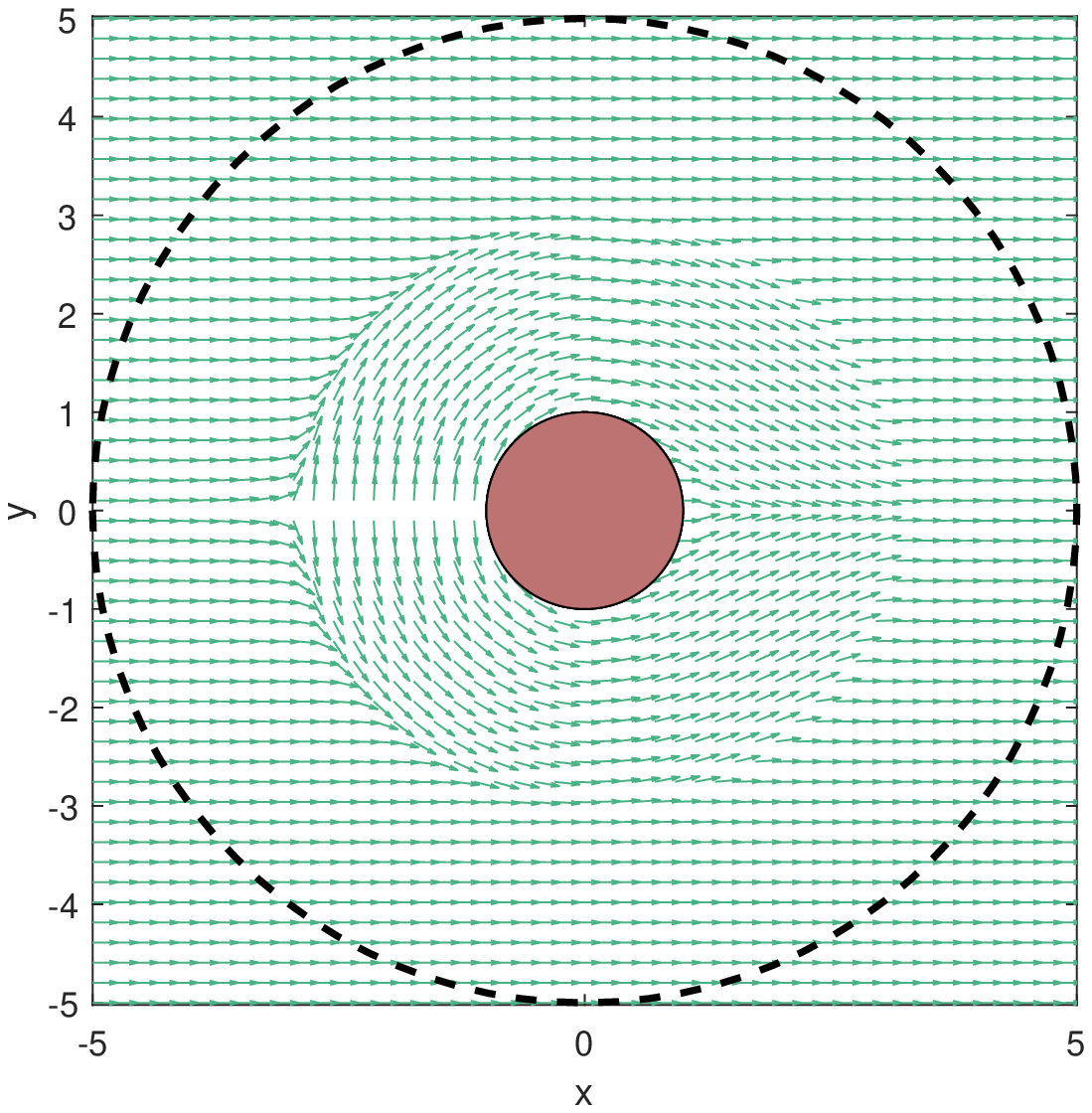}
			\caption{$a = 10$}
		\end{subfigure}
		\caption{Snapshots of a static obstacle CAVF, obtained for varying $a$ values, $\psi_d = 0, ~r_o = 1, ~r_i = 3$}
		\label{fig::Varying_a}	
	\end{figure}
	\begin{figure}[h!]
		\centering
		\begin{minipage}{0.45\textwidth}
			\centering
			\includegraphics[width=3.25in]{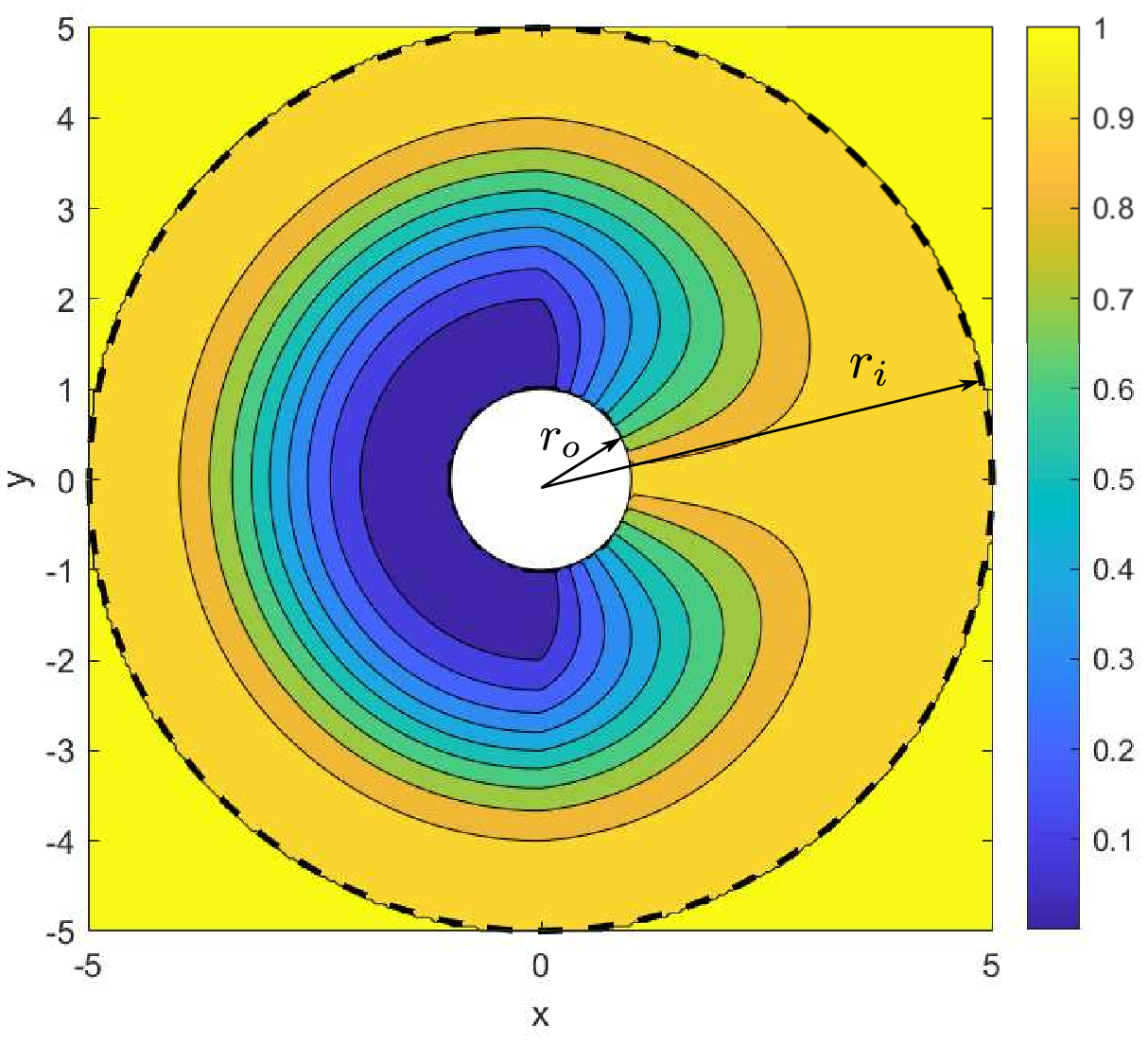}
			\caption{Contours of $\lambda(r,\theta)$, when $a = 1, ~\psi_d = 0, ~r_o = 1, ~r_i = 3$}
			\label{fig::lambda}
		\end{minipage}\hspace{5 mm}
		\begin{minipage}{0.45\textwidth}
			\centering
			\vspace{5mm}
			\includegraphics[width=2.52in]{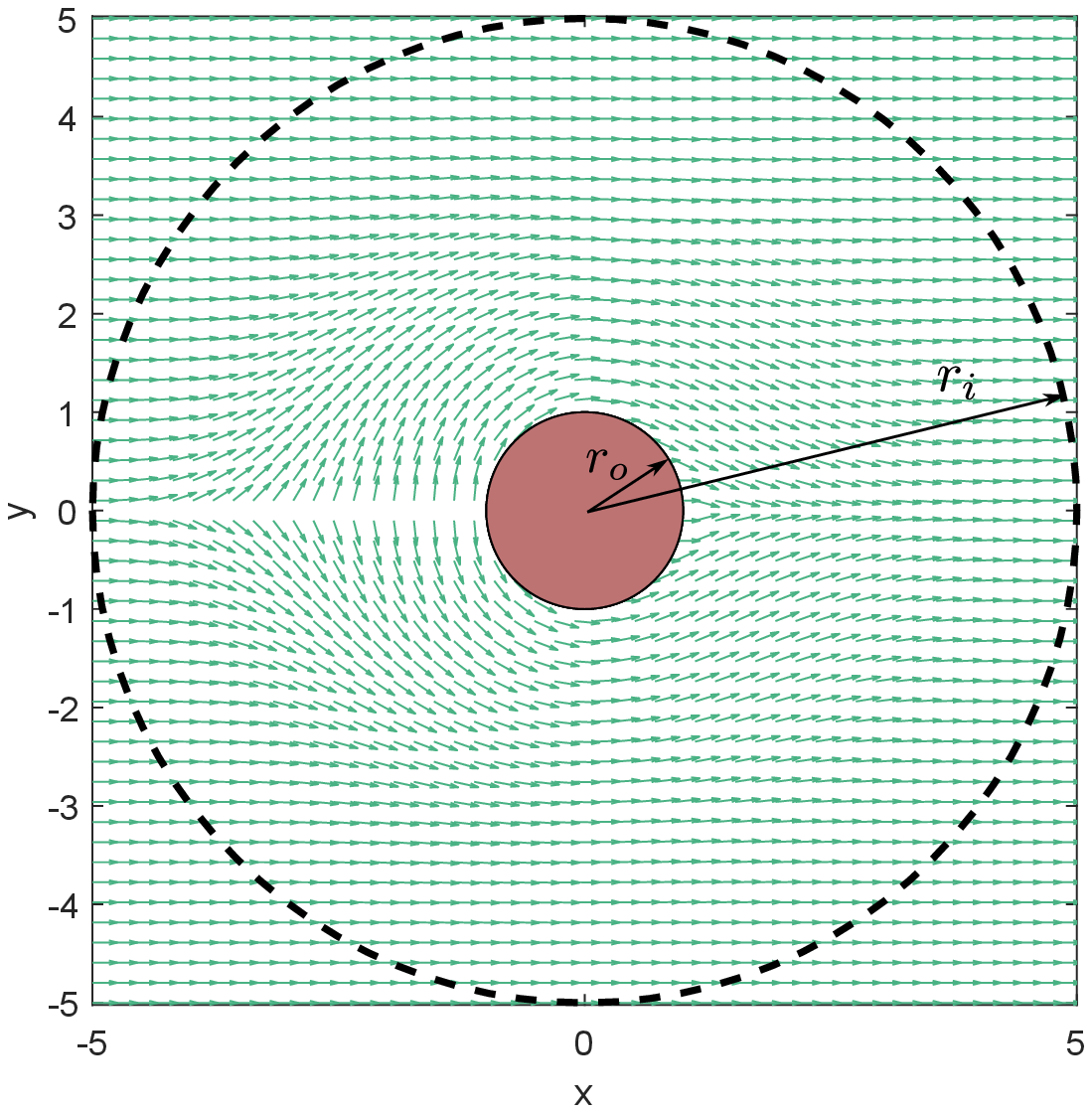}
			\caption{Snapshot of a static obstacle CAVF, obtained using $\lambda(r,\theta)$ with $a = 1, ~\psi_d = 0, ~r_o = 1, ~r_i = 3$ in Eq. \ref{eqn::CAVFstatic}}
			\label{fig::CAVFstatic}
		\end{minipage}
	\end{figure}
	
	Note that the CAVF for a single static obstacle given in \eqref{eqn::CAVFstatic} presents a singularity whenever $\sin\beta = 0$. This singularity defines a switching line, illustrated in Figure \ref{fig::CAVFstatic} by the line determined by $y=0$. The switching line represents the set of agent states from which two actions may be performed to avoid collisions and, as a result, two equivalent solutions to the collision avoidance problem exist. This singularity appears due to the $\sign(\cdot)$ function, which is discontinuous at $0$. It will be shown in the following sections how this issue can be resolved practically. Further, note that the agent's speed does not influence the property given in \eqref{eqn::cavf_def} and, as such, may be changed throughout its motion without compromising the collision avoidance properties of the system.
	
	Next, the results presented for a static obstacle are extended to the case when the obstacle is moving with constant velocity, denoted by $\bm{V_o} = V_o[\cos\theta_o,\sin\theta_o]^T$, where $V_o > 0$ and $\theta_o \in \mathbb{R}$ are the obstacle's speed and heading angle with respect to the inertial $x$-axis, respectively. To avoid collision with this moving obstacle, the agent should not reach the obstacle's surface with a radial speed that is smaller than $|\langle \bm{V_o}, \bm{e_r} \rangle|$. Therefore, an approach that will generate a similar parametrized CAVF around the moving obstacle, instantaneously in time, while trying to maintain as much as possible the original vector field direction, $\psi_d$, in the inertial frame, is proposed. Since the parametrized CAVF of system \eqref{eqn::CAVFstatic} generates a zero radial velocity, applying this in the moving frame will lead to a radial speed of $|\langle \bm{V_o}, \bm{e_r} \rangle|$ in the inertial frame, thus satisfying the requirement for collision avoidance. 
	
	Let $\bm{V_b}$ be the agent's velocity in the relative moving frame, as illustrated in Figure \ref{fig::coordinates}. To generate a similar parametrized CAVF, $^b\bm{V_b}$ must have the same form as $h_s(\bm{p})$ from \eqref{eqn::CAVFstatic}. Therefore, $^b \bm{V_b} = [\dot{r}, r\dot{\theta}]^T$, with:
	\begin{equation}
	\begin{aligned}
	\dot{r} &= -\lambda(r,\theta)V_b\cos\beta \\
	\dot{\theta} &= -\sign(\sin\beta)\frac{1}{r}\sqrt{V_b^2-\dot{r}^2},
	\end{aligned}
	\label{eqn::CAVFmoving}
	\end{equation}
	where $\beta = \angle( {^i\bm{V_b}}, -\bm{e_r} ) = \pi - (\theta - \psi_b)$ as seen in Figure \ref{fig::coordinates}. To maintain continuity between the obstacle-free vector field and the collision avoidance vector field, it is necessary to have:
	\begin{align}
	\psi_b = \atantwo( V\sin\psi_d - V_o\sin\theta_o, ~ V\cos\psi_d - V_o\cos\theta_o ),
	\end{align}
	since the corresponding equations of motion to \eqref{eqn::CAVFmoving} in the inertial frame are obtained from $^i \bm{V} = {^i\bm{V_b}} + {^i \bm{V_o}}$. This relative heading $\psi_b$ takes into account the obstacle's linear motion and preserves the agent's heading in the inertial frame at the interface between obstacle-free motion and collision avoidance. 
	
	Assume next that the agent's speed in the inertial frame is constant, denoted by $V$. Then, the magnitude of the relative velocity, $V_b$, can be found by solving the following equation: $\|{^i \bm{V_b}} + {^i \bm{V_o}}\| = V$. As such, system \eqref{eqn::CAVFmoving} can be implemented and the CAVF for a moving obstacle, $h_d(\cdot)$, is determined as follows:
	\begin{align}
	h_d(\bm{p}) = \bm{R}(\theta) ~{^b\bm{V_b}} + {^i\bm{V_o}},
	\label{eqn::cavf_moving_rotation}
	\end{align}
	where $\bm{R}(\theta)$ is the standard rotation matrix:
	\begin{align}
	\bm{R}(\theta) = \begin{bmatrix} 
	\cos(\theta) & -\sin(\theta) \\
	\sin(\theta) & \cos(\theta)
	\end{bmatrix}.
	\end{align}
	If perfect tracking of this CAVF is achieved, then $\dot{r} = 0$ or $h_d(\bm{p}) = \langle \bm{V_o}, \bm{e_r} \rangle$ whenever $r = r_o$, which implies that the particle will not penetrate the obstacle boundary, thus avoiding collision. A sample vector field is illustrated in Figure \ref{fig::CAVFmoving} for a moving obstacle. The generated CAVF is rotated around the obstacle with angle $\psi_b$, which depends on both the obstacle's and agent's velocity, in order to allow for maneuvers that align with the obstacle motion and still return to the original heading $\psi_d$. Moreover, the CAVF presents strictly positive radial speed almost everywhere on the obstacle surface to avoid the possible collision. As a result, the CAVF would guide an agent found near the obstacle's surface along its direction of motion, by steering it almost in the same direction with the obstacle's velocity to avoid the immediate collision, after which the CAVF would steer the agent around the obstacle for its return to the original course $\psi_d$.
	
	\begin{figure}[h!]
		\centering
		\includegraphics[width=3.25in]{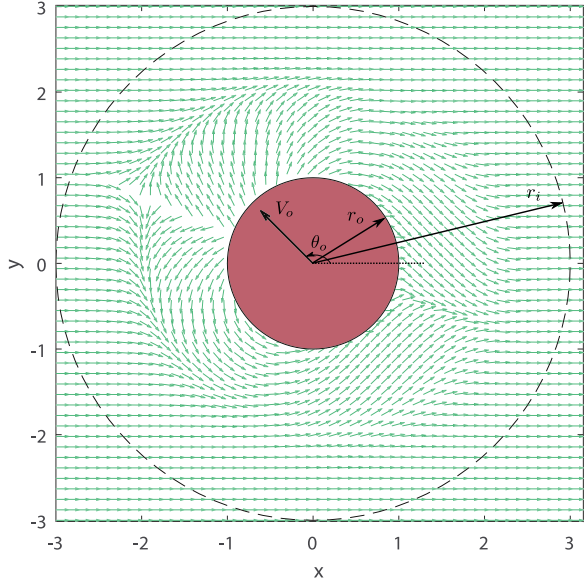}
		\caption{An instantaneous snapshot of the moving obstacle CAVF at $t = 0$, in the inertial frame, obtained using $\lambda(r,\theta)$ with $a = 1, ~\psi_d = 0, ~r_o = 1, ~r_i = 3, ~V_o = 0.9, ~\theta_o = 2.35$ in Eq. \eqref{eqn::CAVFmoving}}
		\label{fig::CAVFmoving}
	\end{figure}

	\subsection{Mixed Collision Avoidance Vector Fields for Multiple Obstacles}
	Suppose next, that multiple obstacles were identified by the agent's sensors and that there is a minimum separation distance between them, denoted by $\delta$. Furthermore, suppose that the collision avoidance parameters are chosen such that the obstacles' radii of influence lead to overlapping CAVFs. In this case, following the streamlines of one CAVF may lead to collision with other obstacles intersecting them. As such, the CAVFs must be mixed in a judicious way so that collision avoidance is still guaranteed. A mixed CAVF is now defined with respect to multiple obstacles.
	
	\begin{definition}
		Suppose there are $n$ obstacles identified by the index set $\mathcal{J}(t) = \{1,2,\dots,n(t)\}$ with overlapping radii of influence, where each one is located at $\bm{p_o}^j \in \mathcal{S}(t)$ and has radius $r_o^j > 0$, for $j \in \mathcal{J}(t)$ such that $\|\bm{p_o}^j - \bm{p_o}^k\| > r^j_o + r^k_o + \delta, ~\forall j,k \in \mathcal{J}(t)$, with $j\neq k$. Their mixed collision avoidance vector field is defined as a spatially dependent vector field $h(\cdot) : \mathbb{R}^2 \to \mathbb{R}^2$ with the property that for all $j \in \mathcal{J}(t)$ there exists $\alpha^j \geq 0$ such that $\langle h(\bm{p_r}), \bm{e_r}^j \rangle \geq \alpha^j$, $\forall \bm{p_r} \in \mathbb{R}^2$ that satisfies $\|\bm{p_r} - \bm{p_o}^j\| = r_o^j$.
		\label{def::MixCAVF}
	\end{definition}
	
	To generate one such CAVF, a method that mixes CAVFs by computing their weighted sum, while preserving magnitude, is proposed in Algorithm \ref{algo::1}. The weights associated with each vector field, denoted by $w^j$, are determined by the agent's distance to each obstacle in $\mathcal{J}(t)$, which is denoted in Algorithm \ref{algo::1} by $\Delta^j, ~\forall j \in \mathcal{J}(t)$. The idea behind this proximity metric comes from the fact that the agent may follow two different motion plans or a combination of the two: its obstacle-free motion plan or the collision avoidance maneuver proposed by an isolated CAVF, depending on how close it is to a particular obstacle surface. As such, if the agent is outside any obstacle's radius of influence, it will continue its original motion plan without any influence from the CAVFs. As the agent enters the radius of influence of one or more obstacles and approaches the obstacles' surfaces, it must perform collision avoidance with respect to all influencing obstacles. Therefore, the proximity metric maps the agent's position into the weight range set $[0,1]$, where a zero weight value is associated with obstacle-free motion whenever the agent is outside the obstacle radius of influence, while a weight value of $1$ is associated with tangential motion around a single isolated obstacle whenever the agent is at that particular obstacle's surface. This mapping is given in lines $2-13$ of Algorithm \ref{algo::1}. Note that here we used the assumption that obstacles cannot overlap in boundary. 
	
	Next, consider the case when the agent is closer to obstacle $\text{ind} \in \mathcal{J}(t)$ than to others. Then, let the value of the weight associated with that obstacle be $\text{val} = w^\text{ind} \geq w^j, ~\forall j \in \mathcal{J}(t)$. If this weight value is within a predefined weight threshold denoted $\epsilon_m$, $\text{val} > \epsilon_m$, where $0 \ll \epsilon_m < 1$, the mixed CAVF will be identical to the CAVF of obstacle $\text{ind}$; otherwise, all weights are normalized. This procedure is illustrated in lines $14-24$ of Algorithm \ref{algo::1} and is required to perform the avoidance maneuver only with respect to one obstacle, since otherwise, mixing CAVFs may not guarantee collision avoidance. 

	\begin{algorithm}
		\caption{Mixing CAVFs}\label{euclid}
		\begin{algorithmic}[1]
			\Statex {\bf Inputs: } $r, ~V, ~\psi_d, ~r_o^j, ~V_o^j, ~\theta_o^j, ~a^j, ~r_i^j ~\forall j \in \mathcal{J}(t)$
			\Statex {\bf Outputs: } $w^j, ~h(\bm{p})$
			\Statex 
			\State $\Sigma_\Delta = 0$
			\For {each obstacle $j$}
				\State Compute $h^j(\bm{p})$ using \eqref{eqn::CAVFstatic} or \eqref{eqn::CAVFmoving}
				\State Compute $\Delta^j = \begin{cases}
					r - r_o^j, & \text{if } r-r^j_i < 0 \\
					-1, & \text{otherwise.}
				\end{cases}$
				\State $\Sigma_\Delta = \Sigma_\Delta + \Delta^j (\Delta^j > 0)$
			\EndFor
			\For {each obstacle $j$}
				\If {$\Sigma_\Delta = \Delta^j$}
					\State $w^j = 1$				
				\Else
					\If {$\Delta^j > 0$}
						\State $w^j = 1 - \Delta^j / \Sigma_\Delta$
					\Else
						\State $w^j = 0$
					\EndIf				
				\EndIf				
			\EndFor
			\State $[\text{val},\text{ind}] = \max_j {w^j}$
			\State $\text{sum}_w = \sum_{j} w^j$
			\If {$\text{val} > \epsilon_m$}
				\State $w^j = 0, ~\forall j \neq \text{ind}$
				\State $w^\text{ind} = 1$
			\Else 
				\If {$\text{sum}_w = 0$}
					\State $w^j = 1, ~\forall j$
				\Else
					\For {each weight $j$}
						\State $w^j = w^j / \text{sum}_w$
					\EndFor
				\EndIf
			\EndIf
			\State $h(\bm{p}) = \sum_{j} w^j h^j(\bm{p})$
		\end{algorithmic}
		\label{algo::1}
	\end{algorithm}

	\begin{proposition}
	Algorithm \ref{algo::1} generates a mixed CAVF. 
	\end{proposition}
	\begin{proof}
		Consider $n(t)$ obstacles with overlapping radii of influence determined by the index set $\mathcal{J}(t) = \{1,2,\dots,n(t)\}$. Suppose that $\bm{p} \in \mathcal{S}(t)$ is such that $\|\bm{p}-\bm{p_o}^j\| \leq r^j_i, ~\forall j \in\mathcal{J}(t)$. 
		Whenever $\|\bm{p}-\bm{p_o}^k\| = r^k_o$ for some $k \in \mathcal{J}(t)$, applying Algorithm \ref{algo::1} will result in $w^k = 1$ and $w^j = 0, ~\forall j\neq k$, which implies that $h(\bm{p}) = h^k(\bm{p})$. Therefore, by Definition \ref{def::MixCAVF}, $h(\bm{p})$ is a mixed CAVF. 
	\end{proof}
	The application of Algorithm \ref{algo::1} will generate a mixed CAVF that will present $n$ singularities or switching lines, as can be seen in the sample mixed CAVF illustrated in Figure \ref{fig::CAVFmix}. In these situations, when the CAVF becomes singular, the flight control system can choose one of the two equivalent solutions to the collision avoidance problem, that is, to turn left or right, by applying a small user-defined correction which will force the agent to follow the CAVF on one side of the switching line. More details about the form of this control correction are provided in the next section. 
	\begin{figure}
		\centering
		\includegraphics[width=3.25in]{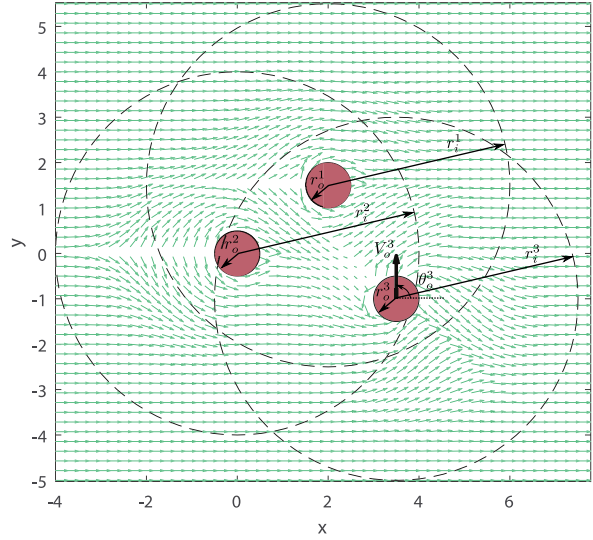}
		\caption{An instantaneous snapshot of the mixed obstacles CAVF at $t = 0$, in the inertial frame, for two static obstacles and one moving obstacle, with the following parameters: $a = 1, ~r_o = 1,~ r_i = 2, ~ V_o = 0.9, ~\theta_o = \pi/2$.}
		\label{fig::CAVFmix}
	\end{figure}
	
	\section{UAV Vector Field Controller}
	\label{sec::Control}
	In this section, an approach for using the CAVFs from Section \ref{sec::CAVF} to generate a solution to the collision avoidance problem presented in Section \ref{sec::Problem} is presented. The approach will be described incrementally, first demonstrating how the local CAVF for a stationary obstacle can be converted into a steering control input. Then it will be shown how the CAVF for a moving obstacle extends naturally the method proposed for a stationary obstacle, using the relative kinematics between the UAV and the obstacle. Lastly, it is shown how a simple tracking controller may be used for the mixed CAVFs while providing at the same time convergence guarantees to the desired collision avoidance maneuvers.
	
	\subsection{Collision Avoidance Controller for a Single Obstacle}
	Consider an UAV moving around a static obstacle of radius $r_o$, located at $\bm{p_o} = [x_o,y_o]^T$. Suppose that the UAV starts performing the collision avoidance maneuver when it is at a distance $r_i > r_o$ from the obstacle geometric center, that is $\| \bm{p}(t_i) - \bm{p}_o \| = r_i$. It is assumed that the UAV's heading matches the desired heading $\psi_d$, if it is at a distance greater than $r_i$, and no steering is required, $u(t) = 0$. Converting the equations of motion of an UAV \eqref{eqn::SysInertial} into polar coordinates results in the following form:
	\begin{equation}
	\begin{aligned}
	\dot{r}(t) &= -V \cos\phi(t), \quad &r(t_i) &= r_i, \\
	\dot{\theta}(t) &= -\frac{V}{r(t)} \sin\phi(t), \quad &\theta(t_i) &= \atantwo{(y(t_i)-y_o, x(t_i)-x_o)},\\
	\dot{\phi}(t) &= u_{s}(t), \quad &\phi(t_i) &= \phi_i, 
	\end{aligned}
	\label{eqn::SysRelativeStatic}
	\end{equation}
	where $r(t) = \norm{\bm{p}(t)-\bm{p_o}}$ is the relative distance between the UAV and the obstacle geometric center, $\theta(t) \in \mathbb{R}$ is the angle between the inertial $x$-axis and the line-of-sight to the obstacle, $\phi(t) \in \mathbb{R}$ is the angle between the line-of-sight and the agent velocity vector, and $u_{s}(t) \in \mathbb{R}$ is the new control input. By inspecting Figure \ref{fig::coordinates}, 
	\begin{align}
	\displaystyle\beta(t) = \angle(^i\bm{V}(t_i),\bm{e_r}(t)), ~\forall t \in [t_i,t_f],
	\label{eqn::beta}
	\end{align}
	where $t_f > t_i$ is the time at which the UAV exits the region of influence of the considered obstacle. Therefore, the following initial condition for $\phi(t)$ is obtained: $\phi_i = \beta(t_i) = \angle{(^i\bm{V}(t_i),\bm{e_r}(t_i))}$. This initial condition guarantees continuity at the maneuver transfer between the high-level plan and the low-level avoidance of the CAVF by making sure that the collision avoidance velocity vector is aligned with the obstacle-free velocity vector. Next, to determine the control input $u_s(t)$ required to achieve the behavior of \eqref{eqn::CAVFstatic}, the following relation, which results from the equivalency between system \eqref{eqn::CAVFstatic} and system \eqref{eqn::SysRelativeStatic}, is considered:
	\begin{align}
	V \cos\phi(t) = \lambda(r,\theta)V \cos\beta(t).
	\label{eqn::CAconstraint}
	\end{align} 
	By differentiating \eqref{eqn::CAconstraint} with respect to time and using the definition of $\beta$ provided in \eqref{eqn::beta}, the following equation is obtained:
	\begin{align}
	\dot\phi(t) = \frac{\dot{\lambda}(t)\cos\big(\theta(t)-\psi_d\big) - \lambda(t) \dot{\theta}(t) \sin\big(\theta(t)-\psi_d\big)}{\sin\phi(t)},
	\label{eqn::phidot}
	\end{align}
	and since $\dot\phi(t) = u_s(t)$, the control input must be equal to the right hand side of \eqref{eqn::phidot}. Here, the notation $\lambda(t)$ is used to denote the implicit time-dependence of $\lambda$ and $\dot\lambda(t)$ is used to denote the total derivative of $\lambda$; in particular, $\lambda(t) := \lambda(r(t),\theta(t))$ and $\dot{\lambda}(t) := \frac{\partial \lambda}{\partial r}\frac{dr}{dt} + \frac{\partial \lambda}{\partial \theta}\frac{d\theta}{dt}$. Moreover, $\dot\phi(t)$ is singular whenever $\sin\phi(t) = 0$, corresponding to the case when the UAV is moving along the line-of-sight to the obstacle and may avoid collision by either turning left or right. As such, the control input that guarantees collision avoidance exists but it is not unique. Hence a small user-defined correction may be performed whenever $\phi(t) = k\pi$, where $k \in \mathbb{Z}$, by replacing $1/\sin\phi(t)$ in \eqref{eqn::phidot} with $k_{\vartheta} = 1/\sin\vartheta$, where $0 < |\vartheta| \ll \pi/2$. This change corresponds to a small deviation from the current path, which will trigger the collision avoidance on one side or another of the obstacle. Thus, the following control input for system \eqref{eqn::SysRelativeStatic} is considered:
	\begin{align}
	u_s(t) = \begin{cases}
	k_{\vartheta} \left( \dot{\lambda}(t)\cos\big(\theta(t)-\psi_d\big) - \lambda(t) \dot{\theta}(t) \sin\big(\theta(t)-\psi_d\big) \right), & \text{if } \phi(t) = k\pi \\
	\displaystyle\frac{1}{\sin\phi(t)}\left(\dot{\lambda}(t)\cos\big(\theta(t)-\psi_d\big) - \lambda(t) \dot{\theta}(t) \sin\big(\theta(t)-\psi_d\big)\right), & \text{otherwise.}
	\end{cases}
	\label{eqn::u_s}
	\end{align}
	The control input given in \eqref{eqn::u_s} can be mapped into inertial coordinates by a simple transformation between system \eqref{eqn::SysInertial} and system \eqref{eqn::SysRelativeStatic} given by $^b\bm{V}(t) = \bm{R}(\theta(t)) {^i\bm{V}(t)}$. From this transformation, the following relation is obtained:
	\begin{align*}
	V\cos\psi(t) &= -V \cos(\phi(t) + \theta(t)), \\
	V\sin\psi(t) &= -V \sin(\phi(t) + \theta(t)),
	\end{align*}
	from which, the inertial control input for collision avoidance is determined: 
	\begin{align}
	u(t) = \dot\phi(t) + \dot\theta(t) = u_s(t) + \dot\theta(t).
	\label{eqn::InertialControlSingleStaticObst}
	\end{align}
	Therefore, if the initial conditions of system \eqref{eqn::SysInertial} are such that \eqref{eqn::CAconstraint} is satisfied, then the application of \eqref{eqn::InertialControlSingleStaticObst} will generate paths that follow perfectly the streamlines of the static obstacle CAVF, as illustrated in Figure \ref{fig::CAVFstatic_stream}. The resulting trajectories show how the UAV should maneuver around the obstacle to avoid collision and return to its original heading course. Note also that the farther the agent moves away from the switching line, the less it is required to maneuver in order to avoid collision, as it can be observed in the top and bottom trajectories. 
	\begin{figure}[h!]
		\centering
		\includegraphics[width=3.25in]{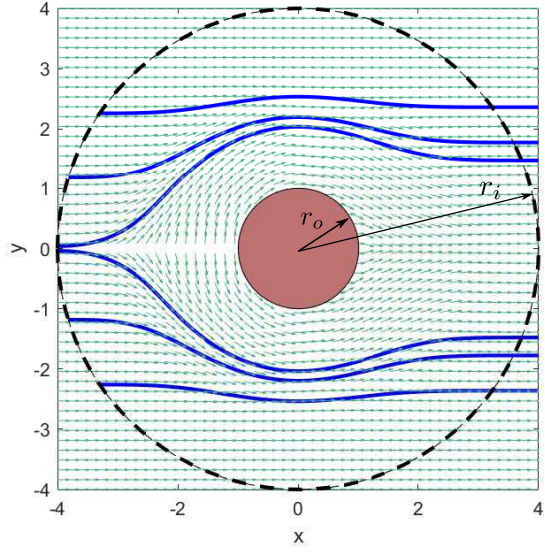}
		\caption{System \eqref{eqn::SysInertial} trajectories (blue) generated using \eqref{eqn::InertialControlSingleStaticObst}, for a CAVF (red vector field) with the following parameters $a = 1, ~r_i = 3, ~r_o = 1, ~\psi_d = 0$.}		
		\label{fig::CAVFstatic_stream}
	\end{figure}	
	
	Next, consider an UAV located near a moving obstacle. Suppose, as before, that the UAV starts performing the collision avoidance maneuver when it reaches a distance $r_i$ from the obstacle geometric center. Converting system \eqref{eqn::SysInertial} into a relative polar coordinate frame results in: 
	\begin{equation}
	\begin{aligned}
	\dot{r}(t) &= -V_b(t) \cos\phi(t), \quad &r(t_i) &= r_i, \\
	\dot{\theta}(t) &= -\frac{V_b(t)}{r(t)} \sin\phi(t), \quad &\theta(t_i) &= \atantwo{(y(t_i)-y_o, x(t_i)-x_o)},\\
	\dot{\phi}(t) &= u_{d}(t), \quad &\phi(t_i) &= \phi_i, 	
	\end{aligned}
	\label{eqn::SysRelativeMoving}
	\end{equation}
	where $V_b(t)$ is the UAV speed in the relative frame of motion and satisfies the following set of equations obtained from mapping system \eqref{eqn::SysRelativeMoving} into the inertial frame and equating it to system \eqref{eqn::SysInertial}:
	\begin{equation}
	\begin{aligned}
	V\cos\psi(t) &= -V_b(t) \cos(\phi(t)+\theta(t)) + V_o\cos\theta_o,\\
	V\sin\psi(t) &= -V_b(t) \sin(\phi(t)+\theta(t)) + V_o\sin\theta_o.
	\end{aligned}
	\label{eqn::Inertial2Polar}
	\end{equation}
	To maintain continuity at the boundary of the CAVF, between the collision avoidance maneuver and obstacle-free motion, the UAV's velocity must satisfy \eqref{eqn::Inertial2Polar} and, as such, the initial condition for $\phi(t)$ is given by:
	\begin{align}
	\phi_i = \atantwo\big(V_o\sin\theta_o-V\sin\psi_d, ~V_o\cos\theta_o-V\cos\psi_d\big) - \theta(t_i).
	\label{eqn::phi_i}
	\end{align}
	Next, to determine the control input $u_d(t)$ required to achieve similar motion patterns that CAVF \eqref{eqn::CAVFmoving} provides, the following relation, which results from the equivalency between system \eqref{eqn::SysRelativeMoving} and system \eqref{eqn::CAVFmoving}, is considered:
	\begin{align}
	V_b(t) \cos\phi(t) = \lambda(r,\theta)V_b(t)\cos\beta(t),
	\label{eqn::CAconstraintmoving}
	\end{align}
	where $\beta(t) = \langle \bm{V_b}(t_i),\bm{e_r}(t) \rangle$. Differentiating \eqref{eqn::CAconstraintmoving} with respect to time and using the definition of $\beta$ given in \eqref{eqn::beta}, determines the following equation:
	\begin{align}
	\dot\phi(t) = \frac{\dot{\lambda}(t)\cos\big(\theta(t)-\psi_b\big) - \lambda(t) \dot{\theta}(t) \sin\big(\theta(t)-\psi_b\big)}{\sin\phi(t)}.
	\label{eqn::phidot_moving}
	\end{align}
	As previously noted, the right-hand side of \eqref{eqn::phidot_moving} is not well-defined whenever $\phi(t) = k\pi$, for $k \in \mathbb{Z}$. Therefore, the following non-singular control input may be used to generate paths that follow perfectly the streamlines of the dynamic CAVF, as long as the initial conditions agree with \eqref{eqn::phi_i}:
	\begin{align}
	u_d(t) = \begin{cases}
	k_{\vartheta} \left( \dot{\lambda}(t)\cos\big(\theta(t)-\psi_b\big) - \lambda(t) \dot{\theta}(t) \sin\big(\theta(t)-\psi_b\big) \right), & \text{if } \phi(t) = k\pi \\
	\displaystyle\frac{1}{\sin\phi(t)}\left( \dot{\lambda}(t)\cos\big(\theta(t)-\psi_b\big) - \lambda(t) \dot{\theta}(t) \sin\big(\theta(t)-\psi_b\big) \right), & \text{otherwise.}
	\end{cases}
	\label{eqn::u_d}
	\end{align}
	
	Further, to have continuity between the control input for obstacle-free motion and for collision avoidance, the agent's desired heading in the moving frame must compensate for the obstacle motion. As such, making the direction of agent's motion in the moving frame correspond to
	\begin{align}
	\psi_b = \phi_i + \theta(t_i) = \atantwo(V_o\sin\theta_o - V\sin\psi_d, V_o\cos\theta_o - V\cos\psi_d),
	\end{align}
	achieves the required continuity between $u_d(t)$ at $r(t) = r_i$ and $u(t)$ at $r(t) \geq r_i$.
	
	Analyzing further the connection between the inertial and polar equations of motion, presented in \eqref{eqn::Inertial2Polar}, the following relation is obtained for the inertial control input:
	\begin{align}
	u(t) = \dot{\psi} =& (u_d(t)+\dot{\theta}) \frac{V_b^2-V_b V_o \cos(\phi+\theta-\theta_o)}{V^2} - \dot{V}_b \frac{V_o \sin(\phi+\theta-\theta_o)}{V^2},
	\label{eqn::InertialControlSingleDynamicObstacle}
	\end{align}
	where
	\begin{align}
	\dot{V}_b = \left(1 + \frac{V_o^2}{V V_b} \sin(\phi+\theta-\theta_o)\sin(\psi-\theta_o)\right)^{-1} \left( \big(u_d(t)+\dot{\theta}\big)\frac{V_o}{V}\big(V_b-V_o\cos(\phi+\theta-\theta_o)\big)\sin(\psi-\theta_o) \right).
	\label{eqn::vbdot}
	\end{align}
	The evolution of the agent's speed in the moving frame is singular whenever the expression $1 + \frac{V_o^2}{V V_b} \sin(\phi+\theta-\theta_o)\sin(\psi-\theta_o)$ tends to $0$. This expression can be equivalently written as:
	\begin{align}
	\bigg(\frac{V}{V_o}\bigg)\bigg(\frac{V_b}{V_o}\bigg) = -\sin{(\phi+\theta-\theta_o)}\sin{(\theta_o-\psi)}.
	\label{eqn::sing}
	\end{align}
	Moreover, applying the sine law in the triangle formed by the vector addition of $\bm{V_b}$ and $\bm{V_o}$, the following relation is obtained:
	\begin{align}
	\frac{V_b}{\sin{(\theta_o-\psi)}} = \frac{V}{\sin{(\phi+\theta-\theta_o)}},
	\end{align}
	which, in view of \eqref{eqn::sing}, results in a parametric condition that leads to the singularity:
	\begin{align}
	\bigg(\frac{V}{V_o}\bigg)^2 = \sin^2{(\phi+\theta-\theta_o)}.
	\label{eqn::sing_cond}
	\end{align}
	Condition \eqref{eqn::sing_cond} will not be true as long as $V>V_o$. Thus, if the agent's speed is higher than the obstacle's speed and if the initial conditions of system \eqref{eqn::SysInertial} are such that \eqref{eqn::CAconstraintmoving} is satisfied, then the application of \eqref{eqn::InertialControlSingleDynamicObstacle} will generate paths that follow perfectly the streamlines of the dynamic obstacle CAVF, as illustrated in the sample simulation trials from Figure \ref{fig::CAVFmoving_stream}.
	\begin{figure}
		\centering
		\begin{subfigure}{0.3\textwidth}
			\centering
			\includegraphics[width=2in]{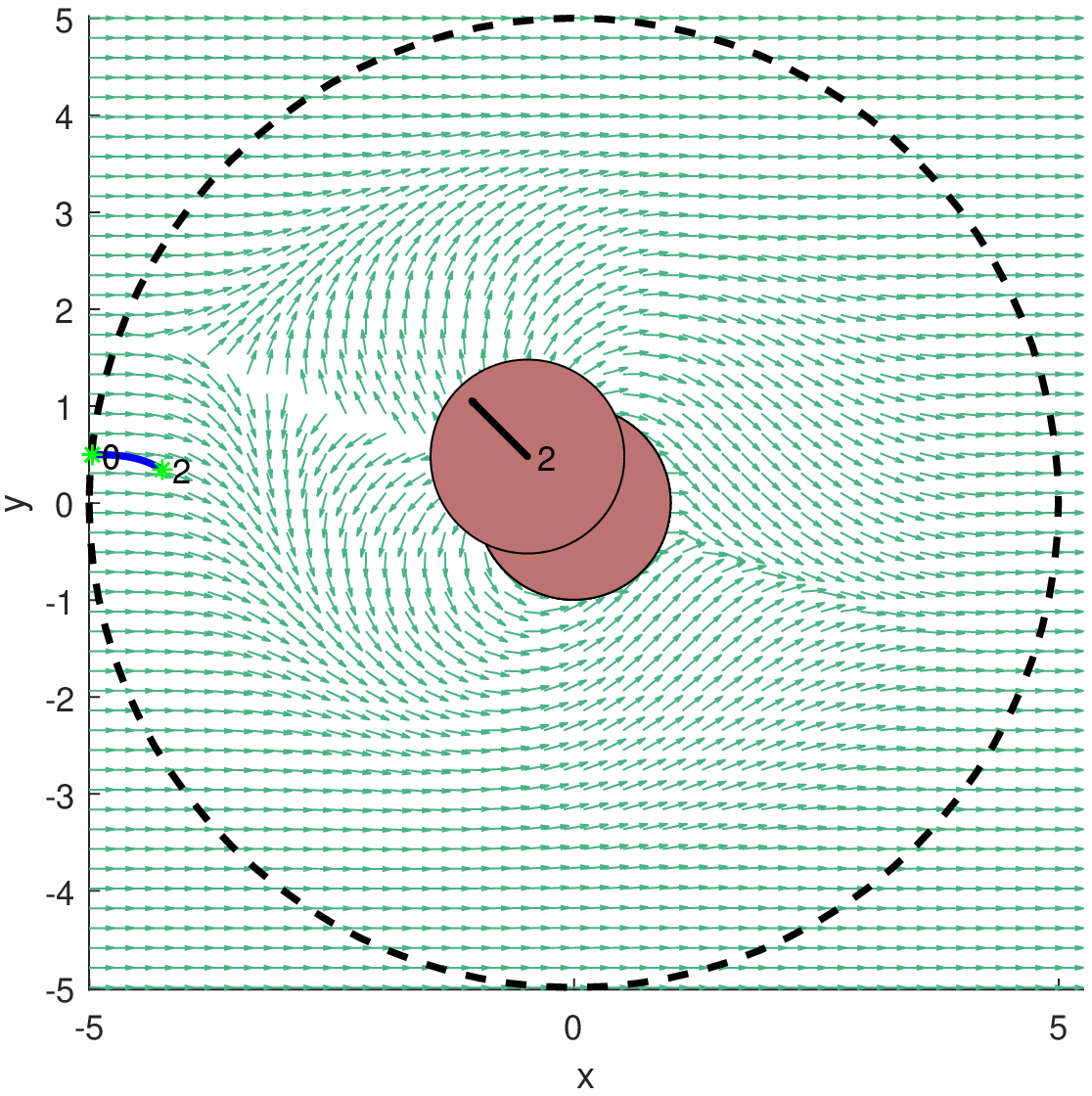}
			\caption{$t \in [0,2]$ sec}
		\end{subfigure}\hspace{1 mm}
		\begin{subfigure}{0.3\textwidth}
			\centering
			\includegraphics[width=2in]{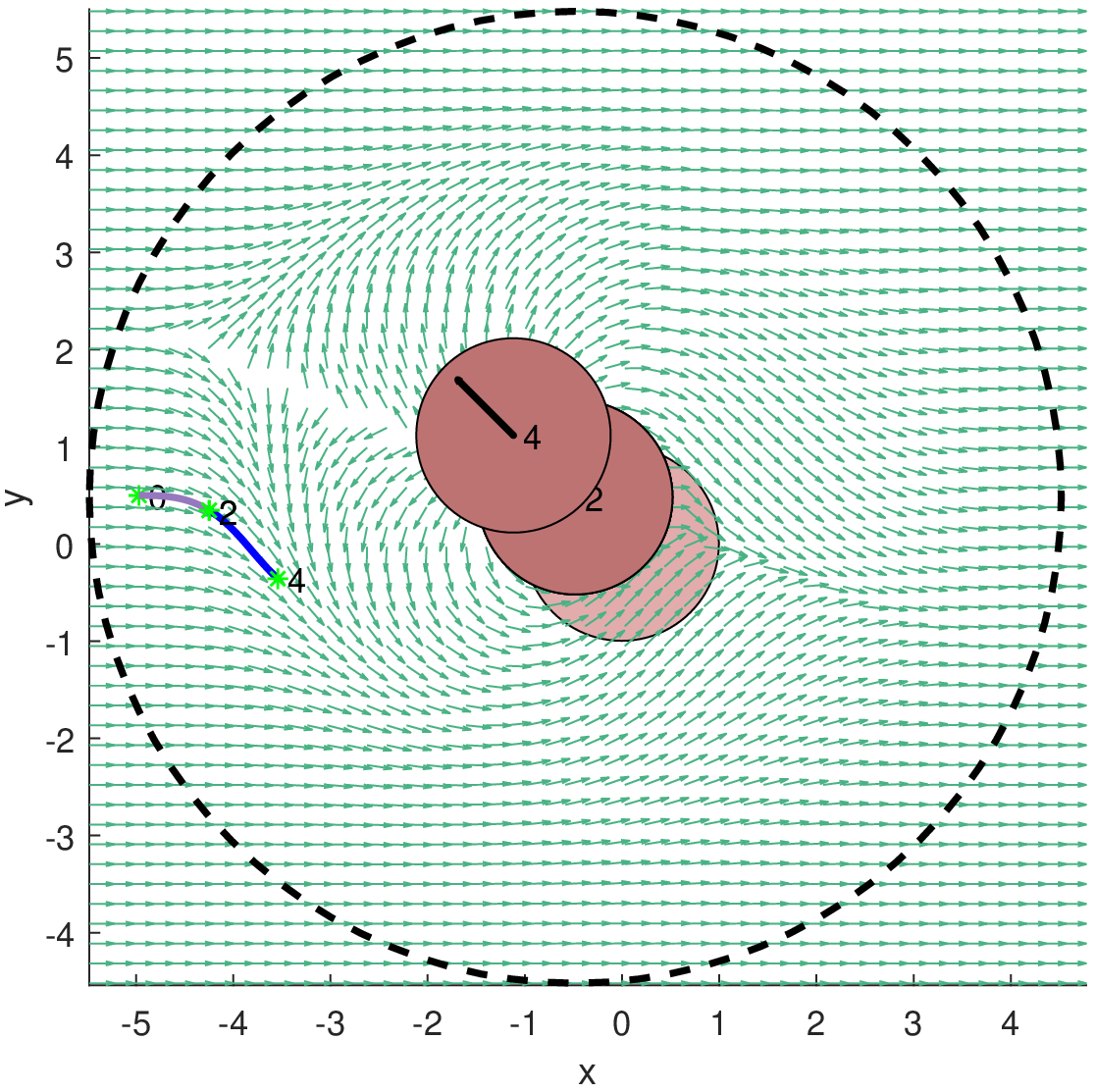}
			\caption{$t \in [2,4]$ sec}
		\end{subfigure}\hspace{1 mm}
		\begin{subfigure}{0.3\textwidth}
			\centering
			\includegraphics[width=2in]{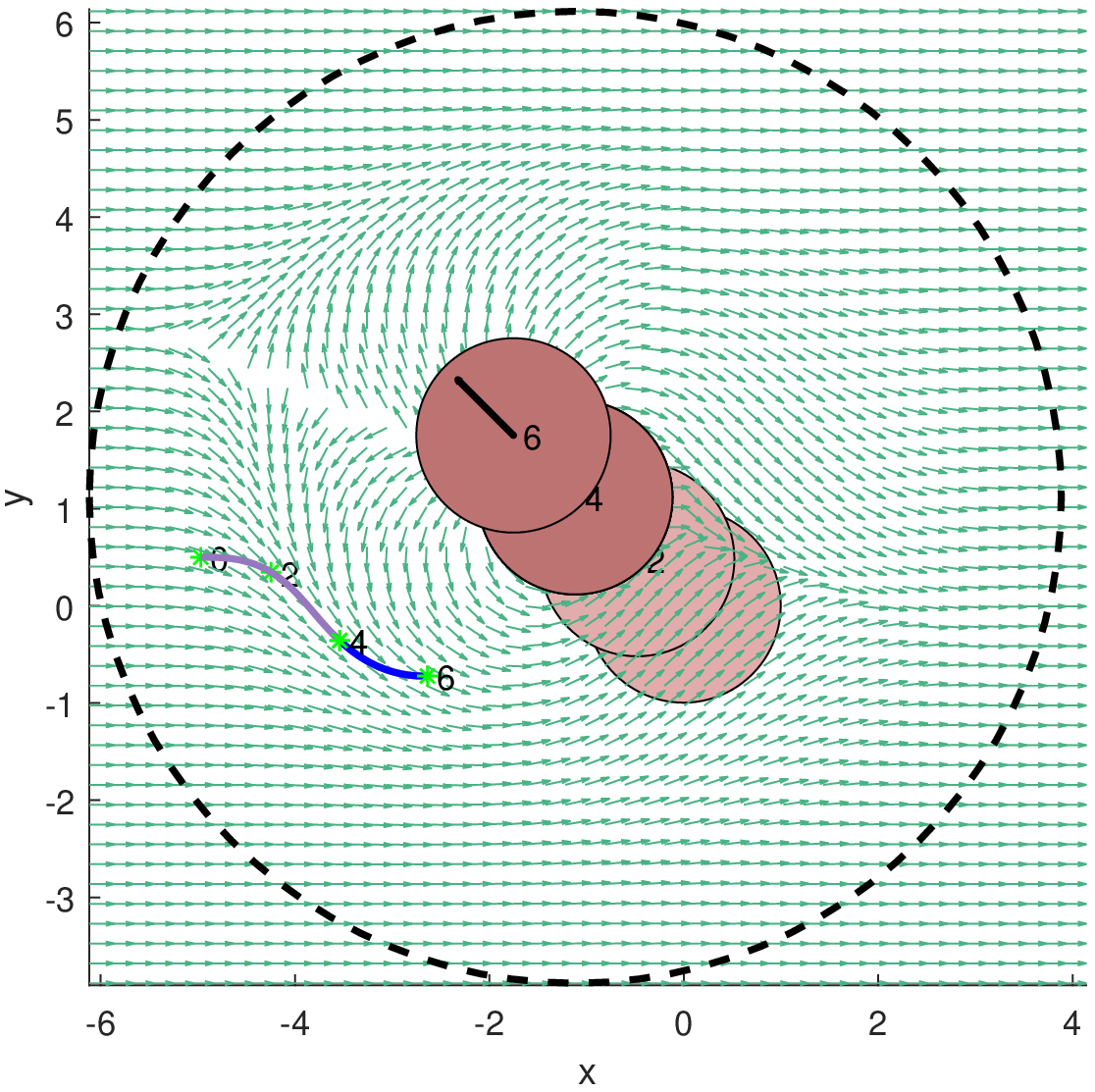}
			\caption{$t \in [4,6]$ sec}
		\end{subfigure}
		\begin{subfigure}{0.3\textwidth}
			\centering
			\includegraphics[width=2in]{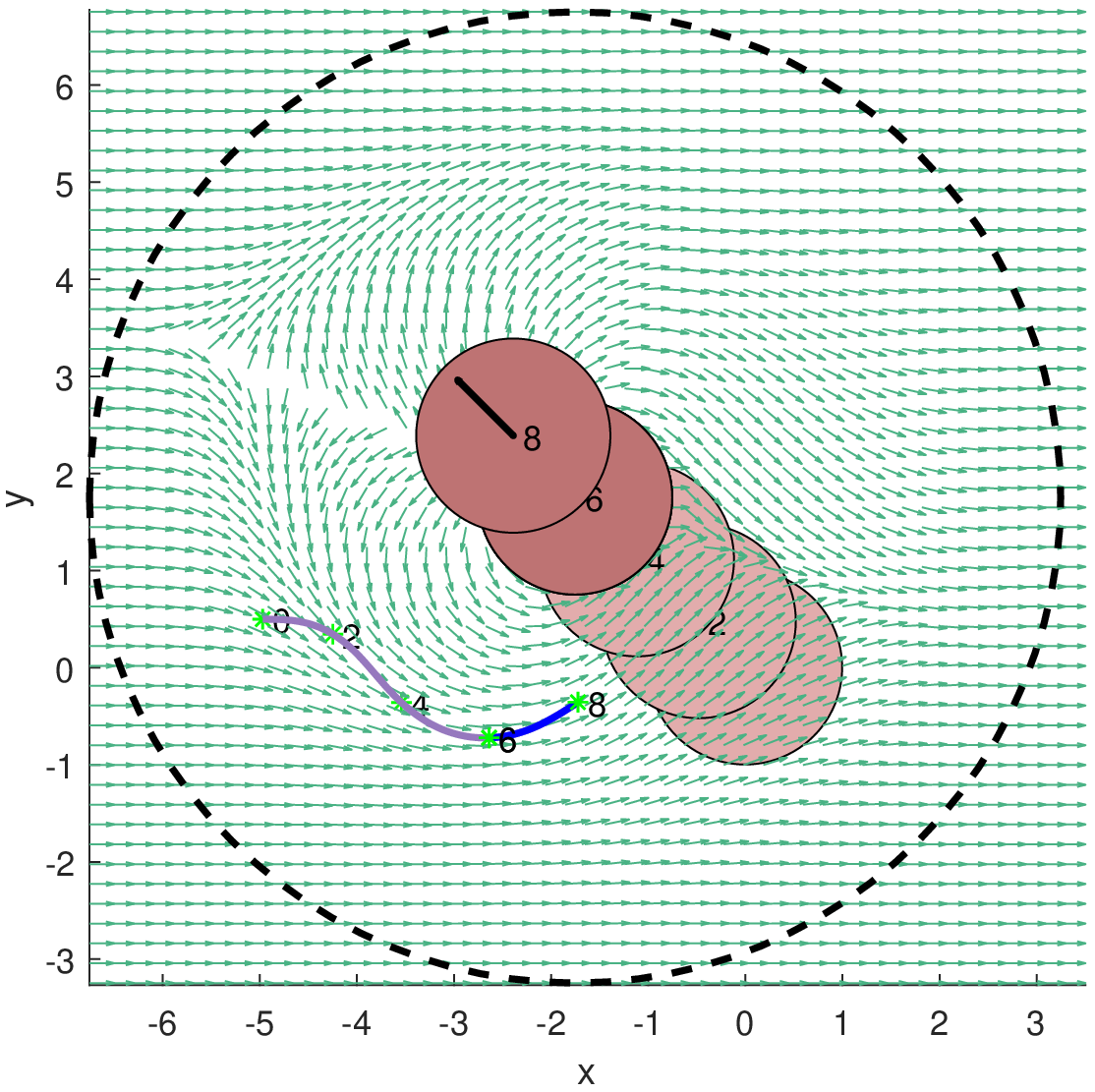}
			\caption{$t \in [6,8]$ sec}
		\end{subfigure}\hspace{1 mm}
		\begin{subfigure}{0.3\textwidth}
			\centering
			\includegraphics[width=2in]{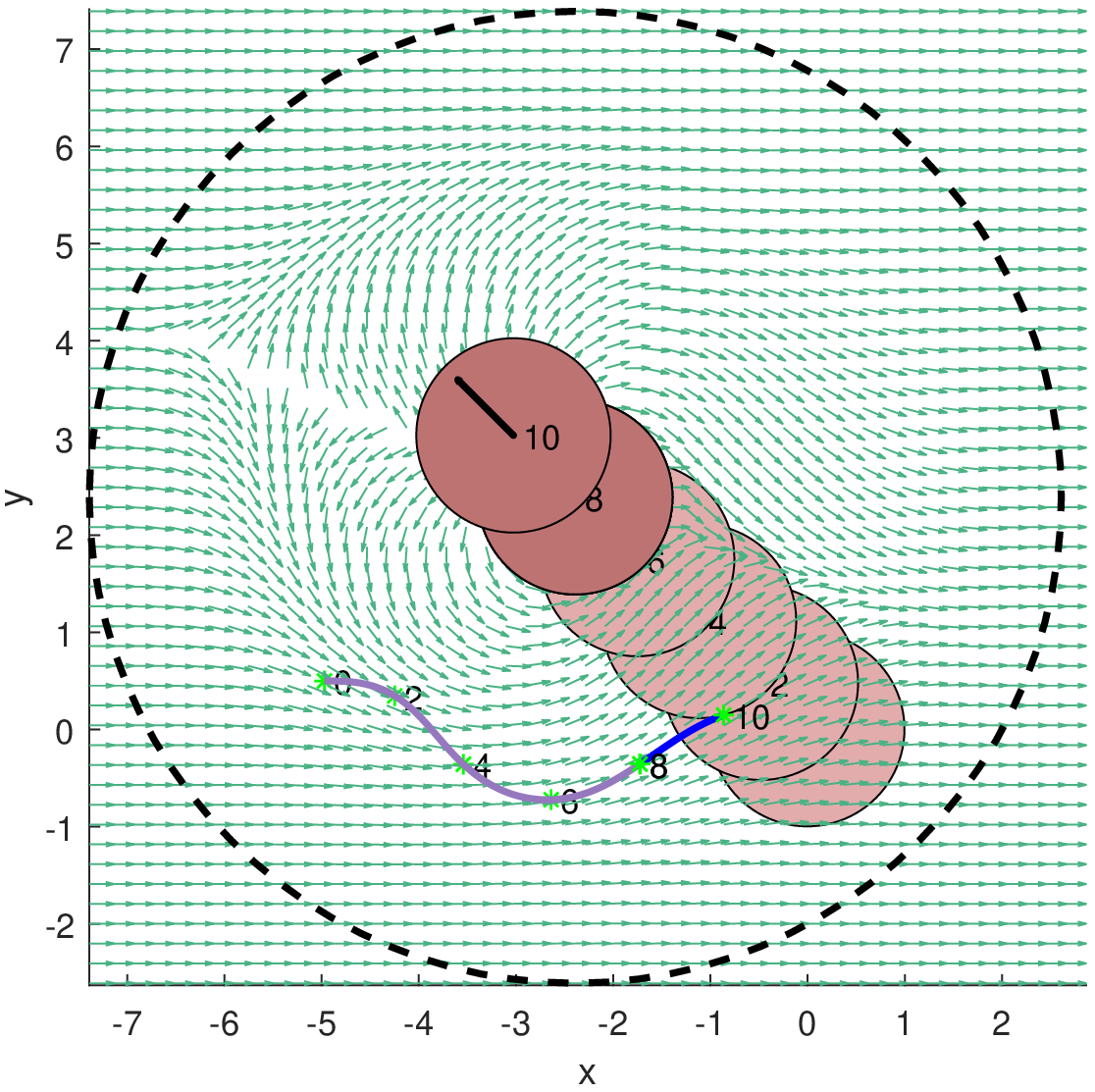}
			\caption{$t \in [8,10]$ sec}
		\end{subfigure}\hspace{1 mm}
		\begin{subfigure}{0.3\textwidth}
			\centering
			\includegraphics[width=2in]{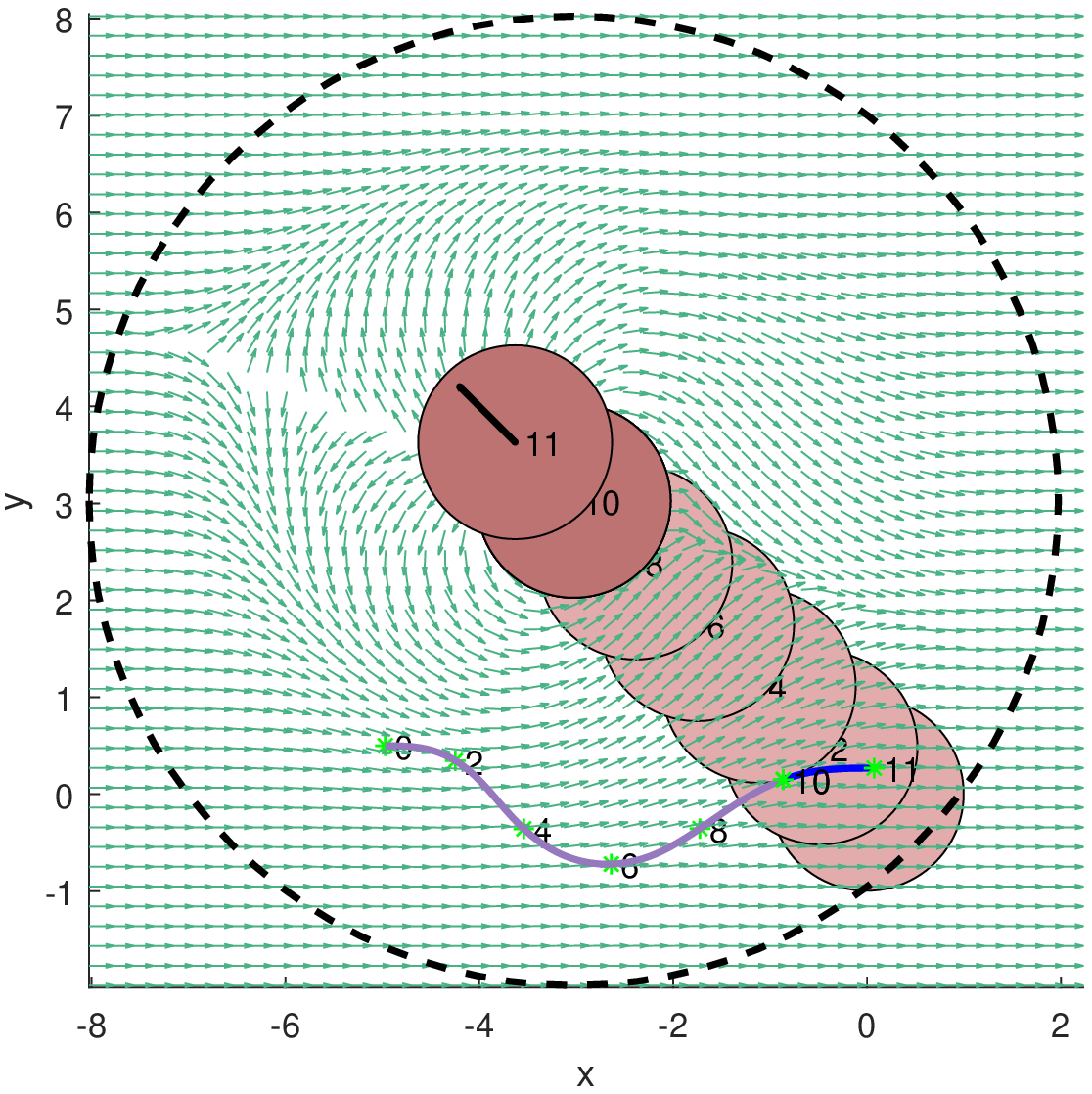}
			\caption{$t \in [10,11]$ sec}
		\end{subfigure}
		\caption{Trajectories of system \eqref{eqn::SysInertial} (blue) generated using input $u(t)$ defined in  \eqref{eqn::InertialControlSingleDynamicObstacle}, for a CAVF (green vector field) with the following parameters $a = 1, ~r_i = 3, ~\psi_d = 0,~ V_o = 0.9, ~\theta_o = 2.35$.}
		\label{fig::CAVFmoving_stream}	
	\end{figure}
	
	\subsection{Collision Avoidance Controller for Multiple Obstacles}
	The previous controllers, given in \eqref{eqn::InertialControlSingleStaticObst} and \eqref{eqn::InertialControlSingleDynamicObstacle}, can be used in the mixing process presented in Algorithm \ref{algo::1} to determine a controller that will try to follow the streamlines of the mixed CAVF, as long as the correct initial conditions are used. 
	\begin{proposition}
		Suppose there are $n$ obstacles identified by the index set $\mathcal{J}(t) = \{1,2,\dots,n(t)\}$ with overlapping radii of influence. Then, an agent trying to follow the streamlines of the corresponding mixed CAVF provided by Algorithm \ref{algo::1}, may use the following control input:
		\begin{align}
		u_m(t) = \sum_{j \in \mathcal{J}(t)} W^j(t) u^j(t),
		\label{eqn::InertialControlMixedObst}
		\end{align} 
		where $u^j(t)$ is the control input for the $j$-th obstacle, which may be of form \eqref{eqn::InertialControlSingleStaticObst} or \eqref{eqn::InertialControlSingleDynamicObstacle}, and $W^j(t)$ is a weight that depends on the evolution of each individual CAVF and on their corresponding mixing weights from Algorithm \ref{algo::1}.
		\label{prop::mult_control}
	\end{proposition}
	\begin{proof}
		See Appendix A.
	\end{proof}

	Moreover, it was noted in Section \ref{sec::CAVF} that the mixed CAVF presents discontinuities resulting from crossing different regions of influence or the weight threshold imposed by Algorithm \ref{algo::1}. Therefore, applying \eqref{eqn::InertialControlMixedObst} to system \eqref{eqn::SysInertial} will not necessarily result in trajectories that follow the streamlines of the mixed CAVF, due to the mentioned discontinuities. As such, an UAV moving around multiple obstacles will require a tracking controller that can achieve convergence to the mixed CAVF in a short amount of time. Specifically, depending on the minimum separation between obstacles, the tracking controller has to guarantee convergence to the vector field within enough time to clear the separation, free from any collision. 
	
	Consider the following tracking controller to be used in the case of an UAV moving around multiple obstacles with overlapping regions of influence:
	\begin{align}
	u_t(t) = -K(\psi(t)-\psi_{\text{ca}}(t)) + u_m(t),
	\label{eqn::TrackingControl}
	\end{align}
	where $K$ is a proportional gain that will be used to enforce tracking convergence, $\psi_{\text{ca}}(t)$ is the mixed CAVF heading and $u_m(t)$ is the controller defined in \eqref{eqn::InertialControlMixedObst}. 
	
	\begin{proposition}
		For a given error tolerance $e_\psi$ in UAV heading convergence to the mixed CAVF, where $0 < |e_\psi| \ll \pi/2$, and a minimum separation distance between obstacles $\delta > 0$, the following proportional gain:
		\begin{align}
		K = \frac{2V(\log{\pi}-\log{e_\psi})}{\delta}
		\label{eqn::TrackingGain}
		\end{align}
		guarantees that controller $u_t(t)$ defined in \eqref{eqn::TrackingControl} results in CAVF tracking.
		\label{prop::gain}
	\end{proposition}
	\begin{proof}
		See Appendix \ref{app::proof}.
	\end{proof}	
	
	Proposition \ref{prop::gain} relates the gain of the tracking controller with a desired heading error tolerance and a desired distance within which convergence is achieved. Therefore, picking a gain $K$ that satisfies \eqref{eqn::GainFormula} will guarantee tracking within the given specifications. 
	
	An illustrative example is shown in Figure \ref{fig::convergence}, in which two obstacles with overlapping radii of influence are considered and the proposed tracking controller is able to follow the direction imposed by the mixed CAVF obtained with Algorithm \ref{algo::1}.
	\begin{figure}
		\centering
		\includegraphics[width=3.25in]{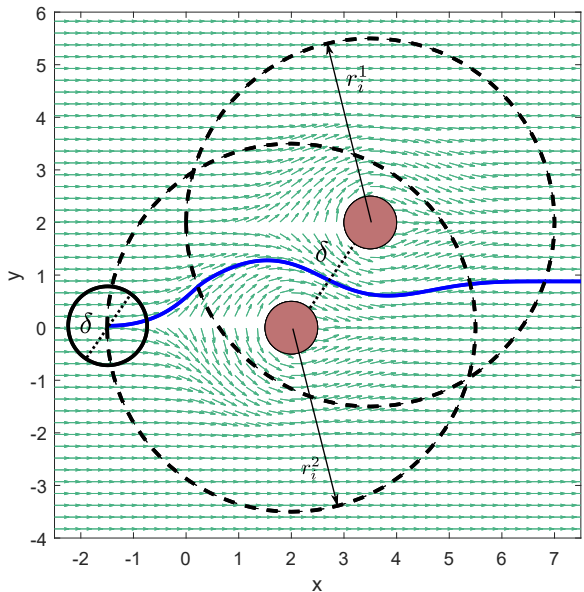}
		\caption{Application of the tracking controller \ref{eqn::TrackingControl} to system \ref{eqn::SysInertial}, where the gain $K$ depends on the separation distance, $\delta$, between two static obstacles with overlapping radii of influence, $r_i^1$ and $r_i^2$.}
		\label{fig::convergence}
	\end{figure}

	Finally, it has been observed in practice that other weights may be used for the mixing controller \ref{prop::mult_control} and collision avoidance is still preserved regardless of the type of obstacles that are encountered. For example, one may use directly the weights in Algorithm \ref{algo::1}, $w^j$, to compose a simpler controller with similar performance to the one presented in \ref{prop::mult_control}. 

	\section{Simulation Results}
	\label{sec::sims}
	In this section, three different simulation scenarios with real world applications are used to analyze the performance of the proposed methodology for collision avoidance. In the first scenario, a simulation with an UAV moving through an environment with multiple scattered static obstacles is presented. This scenario can be representative, for example, for an UAV moving at a constant height through a forest in which every tree is modeled as a static cylinder. In the second scenario, a simulation of an UAV moving through an environment populated with multiple moving obstacles, similar to a busy airspace populated with other UAVs, is considered. Lastly, a simulation of an UAV navigating through a complex environment, populated by both static and moving obstacles, is presented.
	
	\subsection{Collision Avoidance Scenario $1$: Navigation through a densely populated workspace with static obstacles}
	Consider an UAV modeled by system \eqref{eqn::SysInertial} moving through a workspace with multiple static obstacles. In a realistic setting, for example, this workspace can represent a small forest patch that contains $12$ trees, where each tree is modeled as a static circular obstacle if the UAV is moving at the same altitude. The UAV's control objective is to pass through the workspace while maintaining an initial Eastward direction, $\psi_d = 0$, and while avoiding any tree collision. To this end, the proposed methodology for creating a CAVF and tracking its streamlines is applied. 
	
	The results of the simulation are illustrated in Figure \ref{fig::forest} for the following UAV initial conditions and parameters, respectively: $r_s = 12~ \mathrm{m}$, $x_i = 0$, $y_i = 1.3$, $\psi_i = \psi_d = 0, ~V = 1~\mathrm{m}/\mathrm{s}$, $a= 1$ and $r_i^j = 2 ~\mathrm{m}, ~\forall j \in \mathcal{J}(t)$. The trajectory was obtained using the control input presented in \eqref{eqn::TrackingControl} with a gain $K = 25$, by tracking the mixed CAVF illustrated in Figure \ref{fig::forest_CAVF}. The gain was chosen such that it satisfies the condition for the minimum separation between obstacles, which in the given example is $\delta = 0.516$, and for the heading convergence error $e_\psi = 0.01$. The mixed CAVF was obtained by applying Algorithm \ref{algo::1}. As the agent moves through the forest patch, it has to adapt to its workspace and switch from avoiding one obstacle to another. These switches can be seen by the short discontinuities in the steering control, illustrated in Figure \ref{fig::forest_controls} by the short abrupt changes in the required steering rates, whenever the agent enters or leaves the region determined by an obstacle's radius of influence. The controller defined in \eqref{eqn::TrackingControl} is able to make the UAV's heading (illustrated by the black line in Figure \ref{fig::forest_controls}) follow closely the mixed CAVF heading (illustrated by the green dashed line in Figure \ref{fig::forest_controls}). Changing the gain $K$ would result in a better convergence but more demand from the UAV's actuators, by requesting more angular speed which would increase the spikes in $\dot{\psi}$. 
	\begin{figure}
		\centering
		\begin{minipage}{1\textwidth}
			\centering
			\includegraphics[width=6.5in]{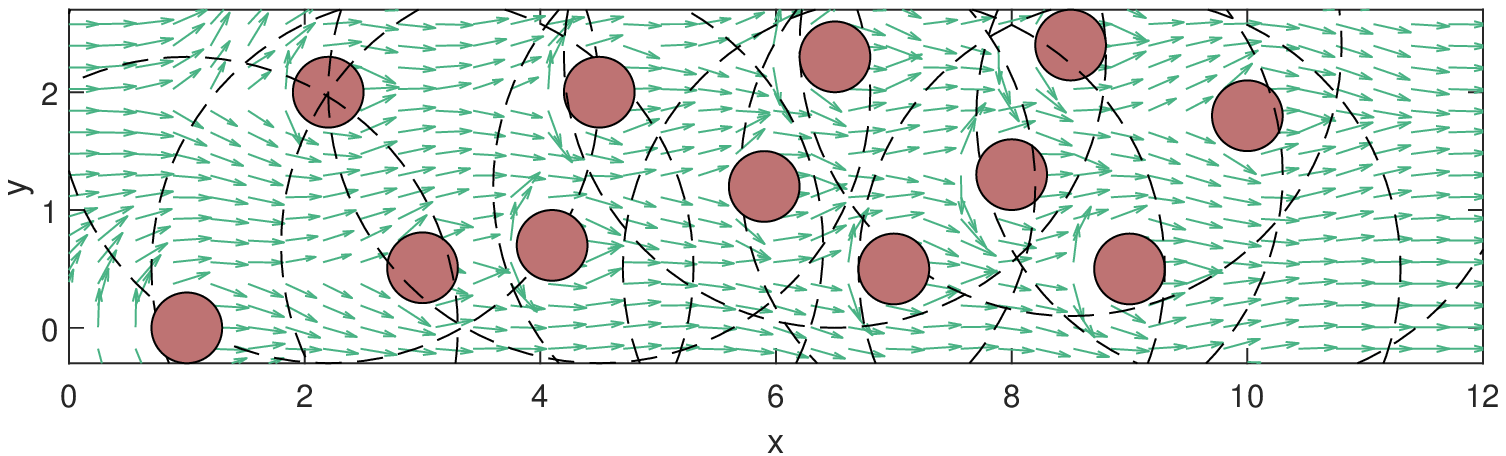}
			\caption{The resulting mixed CAVF for the given small forest patch and collision avoidance parameters.}
			\label{fig::forest_CAVF}
		\end{minipage}
		\begin{minipage}{1\textwidth}
			\centering
			\includegraphics[width=6.5in]{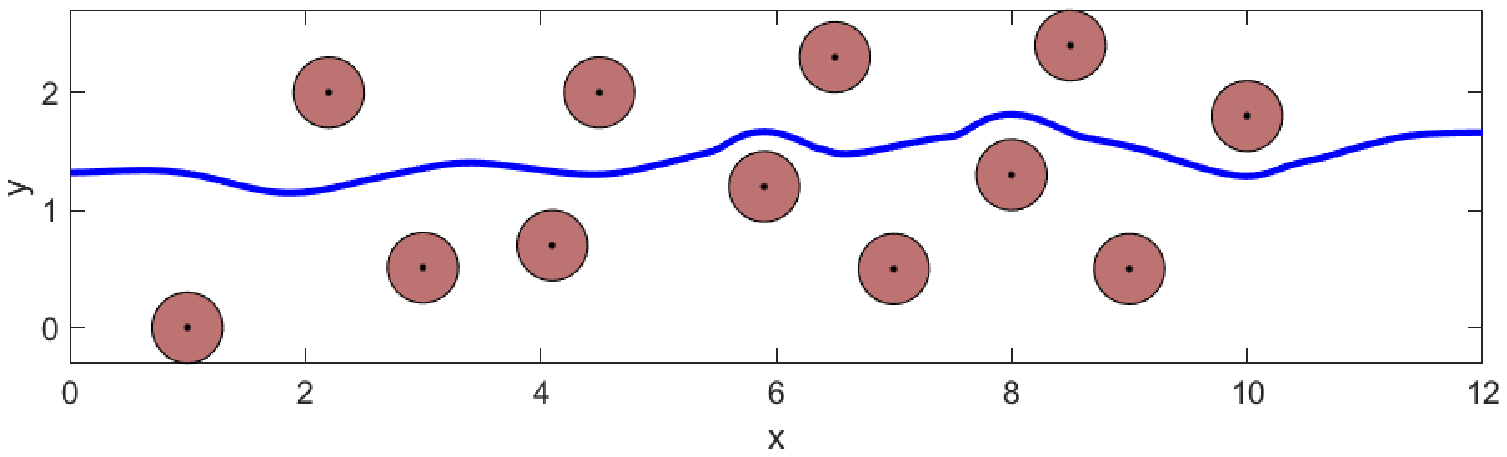}
			\caption{Collision avoidance for an UAV (blue line) moving through a forest patch with trees modeled as circular static obstacles of radius $r_o^j = 0.3~\mathrm{m}, ~\forall j \in \mathcal{J}(t)$.}
			\label{fig::forest}
		\end{minipage} 
	\end{figure}
	
	\begin{figure}[h!]
		\centering
		\includegraphics[width=6.5in]{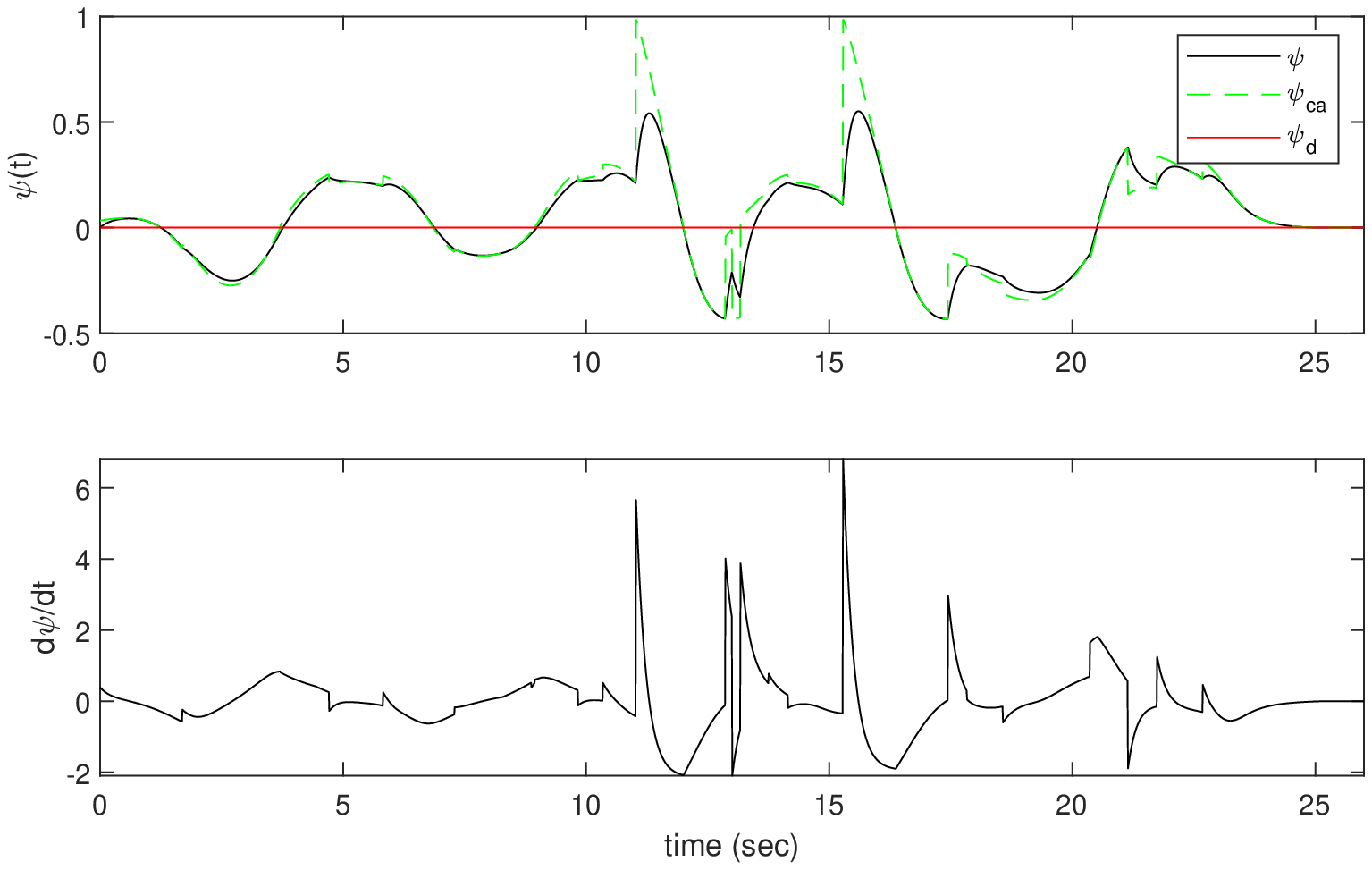}
		\caption{Steering rate controls with resulting UAV headings.}
		\label{fig::forest_controls}
	\end{figure}
	
	\subsection{Collision Avoidance Scenario $2$: Navigation through a densely populated workspace with moving obstacles}
	\label{sec::Sim2}
	Consider next the case when the UAV is moving through a workspace populated by multiple moving obstacles. In a realistic setting, this type of workspace could be representative for a high-traffic region of the airspace, where multiple UAVs are trying to perform cooperative or non-cooperative tasks, such as search and rescue, pay-load delivery, or surveillance. For example, suppose the workspace contains $5$ other moving UAV's, where each one is modeled as a moving circular obstacle of radius $r_o^j = 0.3~\mathrm{m}, ~\forall j\in\mathcal{J}(t)$. Each obstacle is moving with constant velocity, at different speeds and headings. The UAV's objective is to move through this workspace while avoiding collision with any incoming obstacles and while maintaining an Eastward general heading, $\psi_d = 0$. 
	
	The CAVF for this problem is generated by applying the methodology presented in section \ref{sec::CAVF} for multiple moving obstacles with overlapping radii of influence. Therefore, the controller from Section \ref{sec::Control} is used to track the resulting CAVF. The results of the simulation are illustated in Figure \ref{fig::Airspace} for the following UAV initial conditions and parameters, respectively: $x_i = -3.3, ~y_i = 0, ~\psi_i = \psi_d = 0, ~ V = 1~\mathrm{m}/\mathrm{s},~ a = 1$ and $r_i^j = 3, ~\forall j \in \mathcal{J}(t)$. The trajectory was obtained using the control input given in \eqref{eqn::TrackingControl} with a time-varying gain $K(t) = 11.5\delta(t)^{-1}$, where $\delta(t) = \min_{j\in\mathcal{J}(t)}{\norm{r^j(t)-r_o^j}}$ is the minimum distance to an obstacle at time $t$ and where $r^j(t) = \norm{\bm{p}-\bm{p_o}^j}$ is the distance between the agent's position and the obstacle $j$'s position. The heading convergence error is set to be $e_\psi = 0.01$. The simulation shows an agent performing multiple collision avoidance maneuvers around the moving obstacles. Looking at the steering rate controls in Figure \ref{fig::Airspace_controls}, it can be noted that the presented controller is able to track with accuracy the desired CAVF heading. The peaks in the profile of steering rate v/s time determine different control authority switches for collision avoidance, depending on the obstacle proximities. 
	\begin{figure}
		\centering
		\begin{subfigure}{0.3\textwidth}
			\centering
			\includegraphics[width=2in]{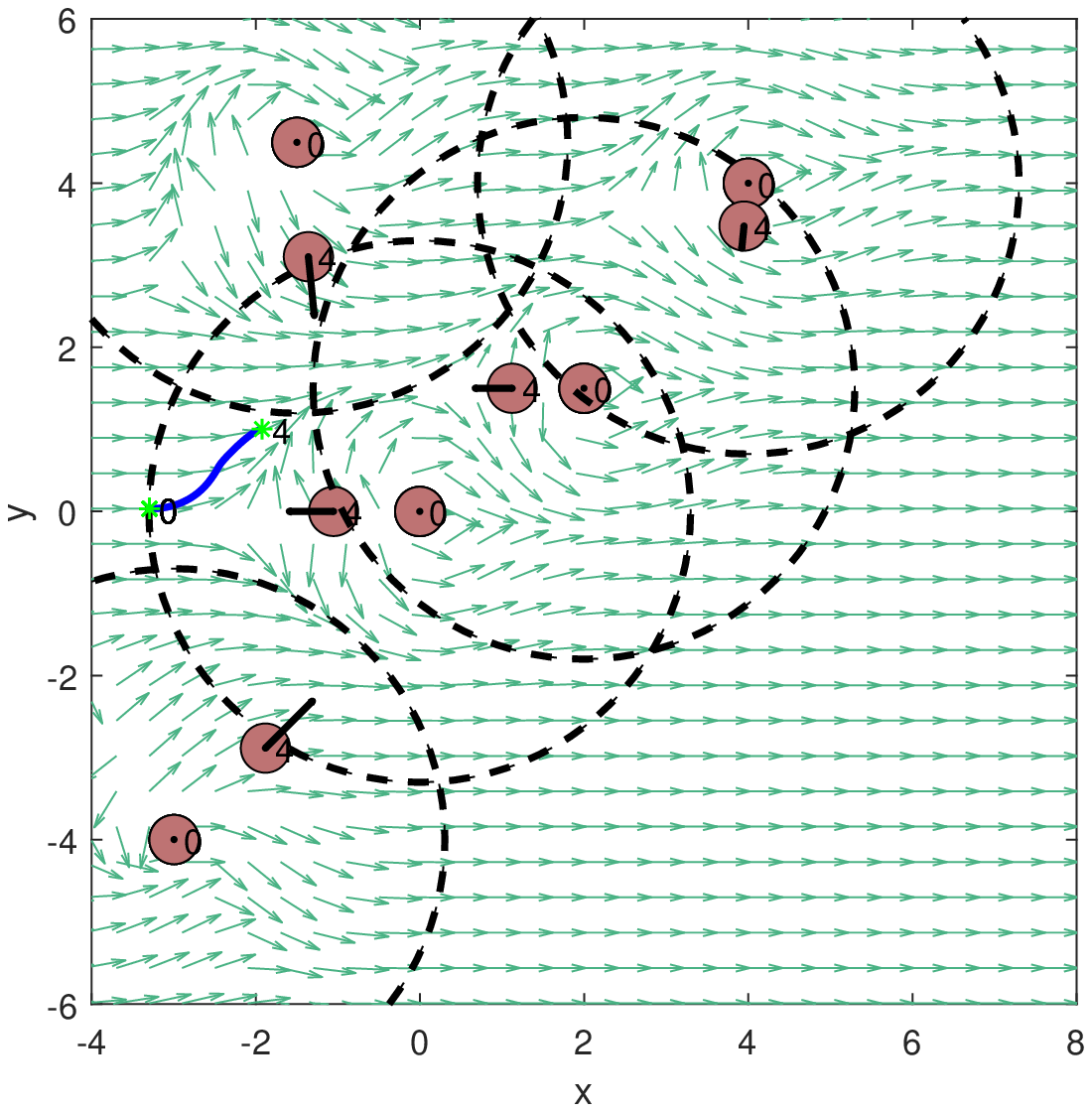}
			\caption{$t \in [0,4]$ sec}
		\end{subfigure}\hspace{1 mm}
		\begin{subfigure}{0.3\textwidth}
			\centering
			\includegraphics[width=2in]{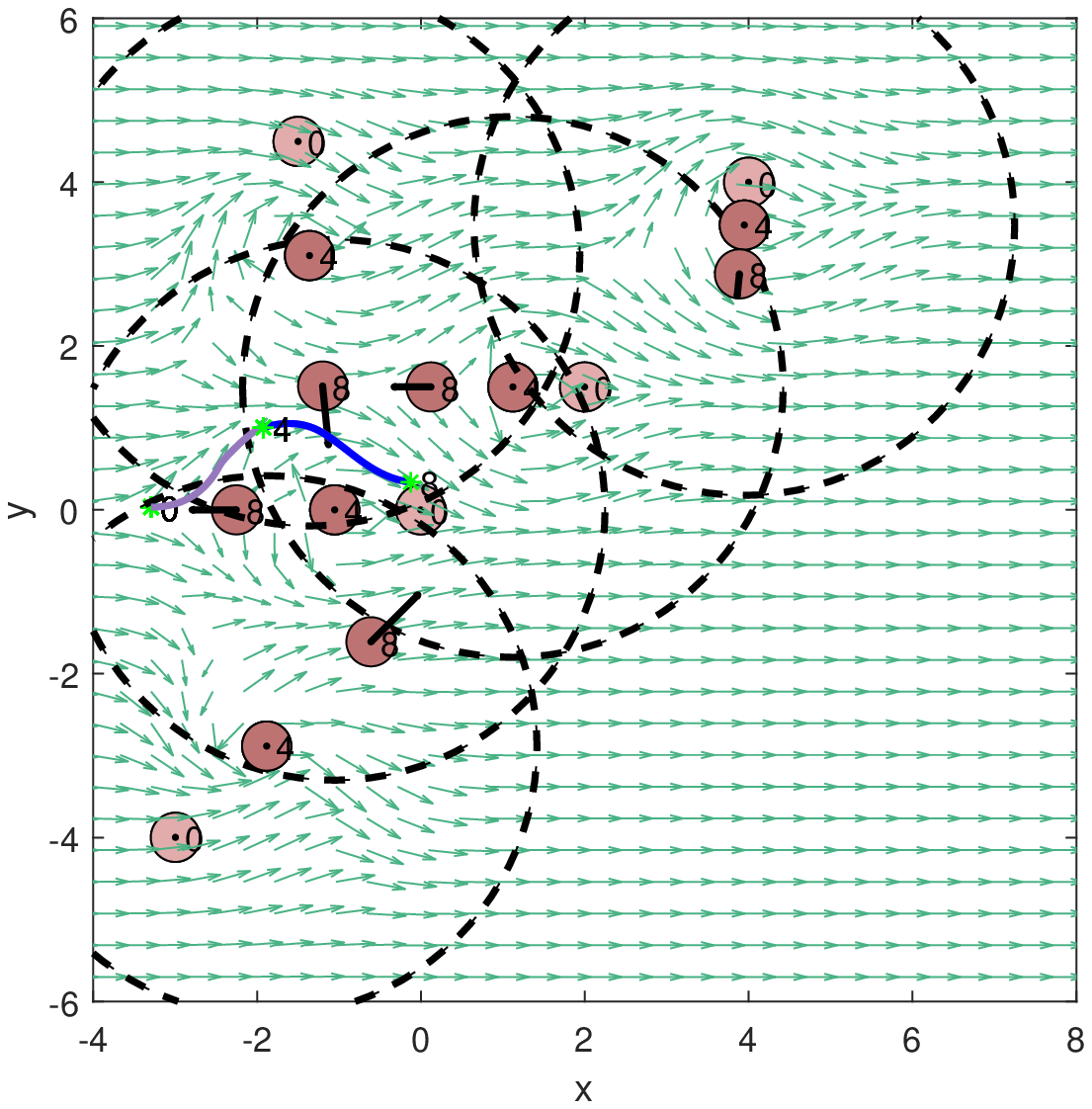}
			\caption{$t \in [4,8]$ sec}
		\end{subfigure}\hspace{1 mm}
		\begin{subfigure}{0.3\textwidth}
			\centering
			\includegraphics[width=2in]{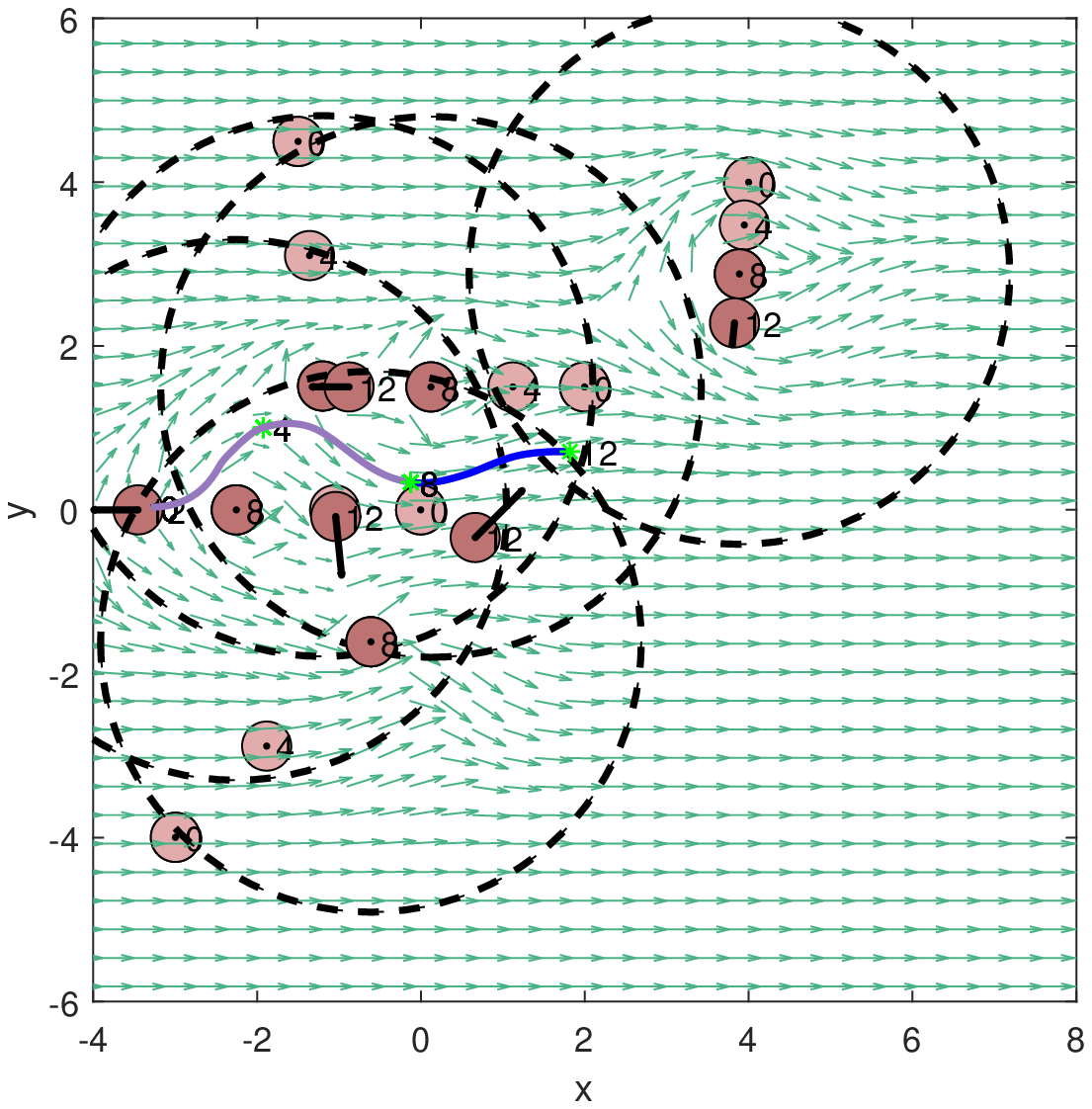}
			\caption{$t \in [8,12]$ sec}
		\end{subfigure}
		\begin{subfigure}{0.3\textwidth}
			\centering
			\includegraphics[width=2in]{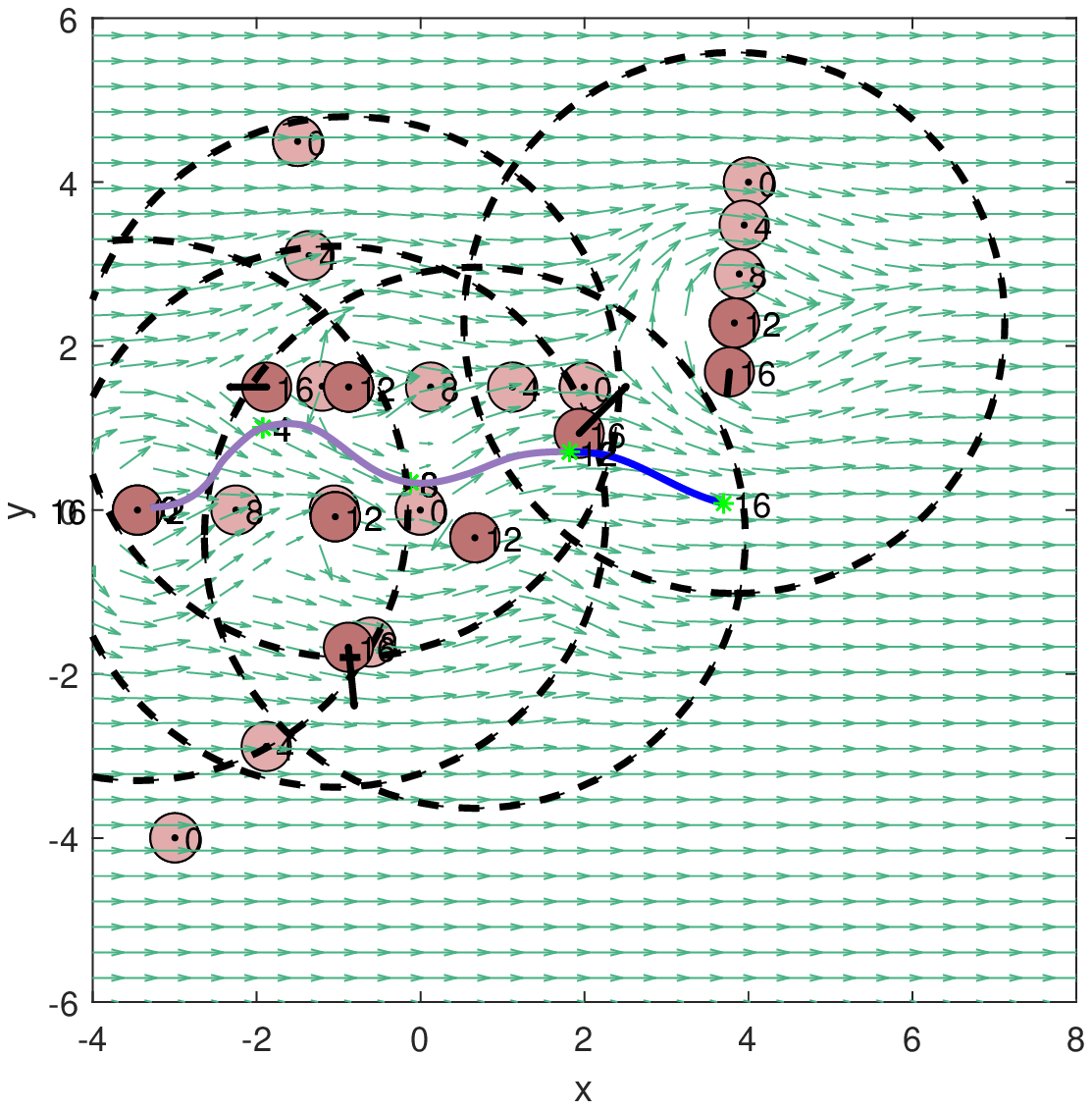}
			\caption{$t \in [12,16]$ sec}
		\end{subfigure}\hspace{1 mm}
		\begin{subfigure}{0.3\textwidth}
			\centering
			\includegraphics[width=2in]{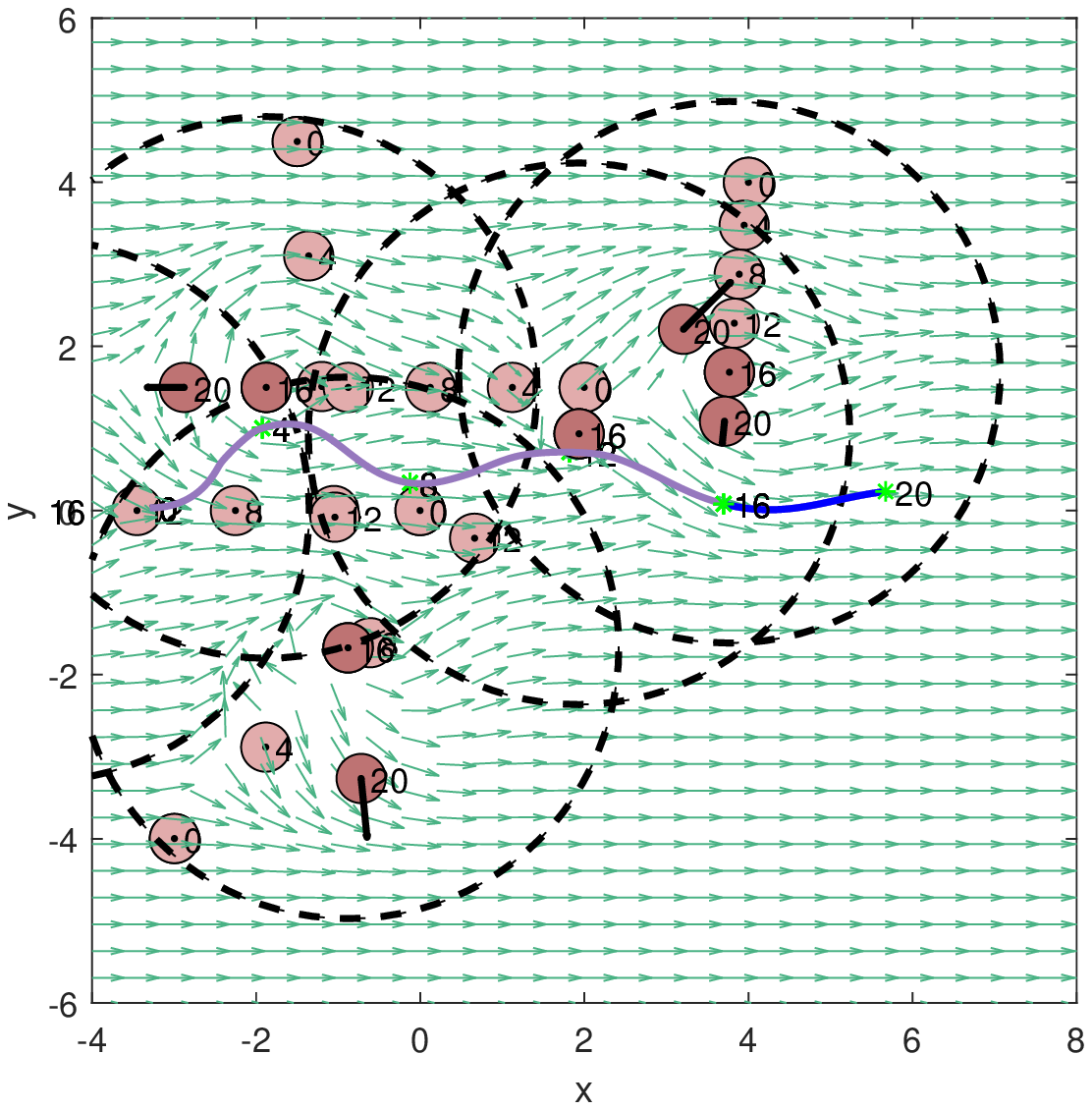}
			\caption{$t \in [16,20]$ sec}
		\end{subfigure}\hspace{1 mm}
		\begin{subfigure}{0.3\textwidth}
			\centering
			\includegraphics[width=2in]{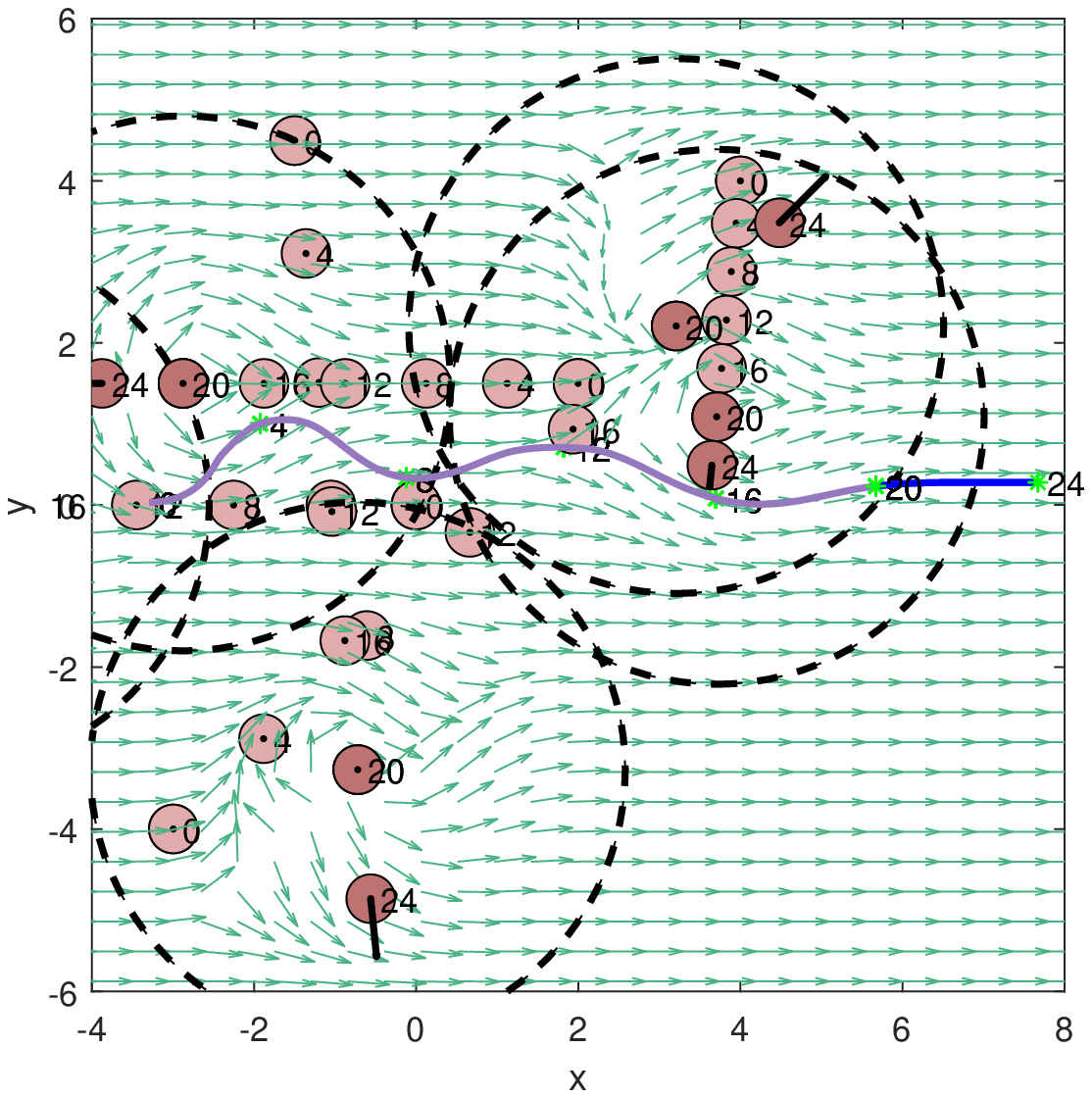}
			\caption{$t \in [20,24]$ sec}
		\end{subfigure}
		\caption{Trajectories of system \eqref{eqn::SysInertial} (blue) generated using input $u(t)$ defined in  \eqref{eqn::InertialControlSingleDynamicObstacle}, for a CAVF (green vector field) with the following parameters $a = 1, ~r_i = 3 ~\mathrm{m}, ~\psi_d = 0,~ V_o = 0.9 ~\mathrm{m}/\mathrm{s}, ~\theta_o = 2.35$.}
		\label{fig::Airspace}	
	\end{figure}
	\begin{figure}
		\centering
		\includegraphics[width=6.5in]{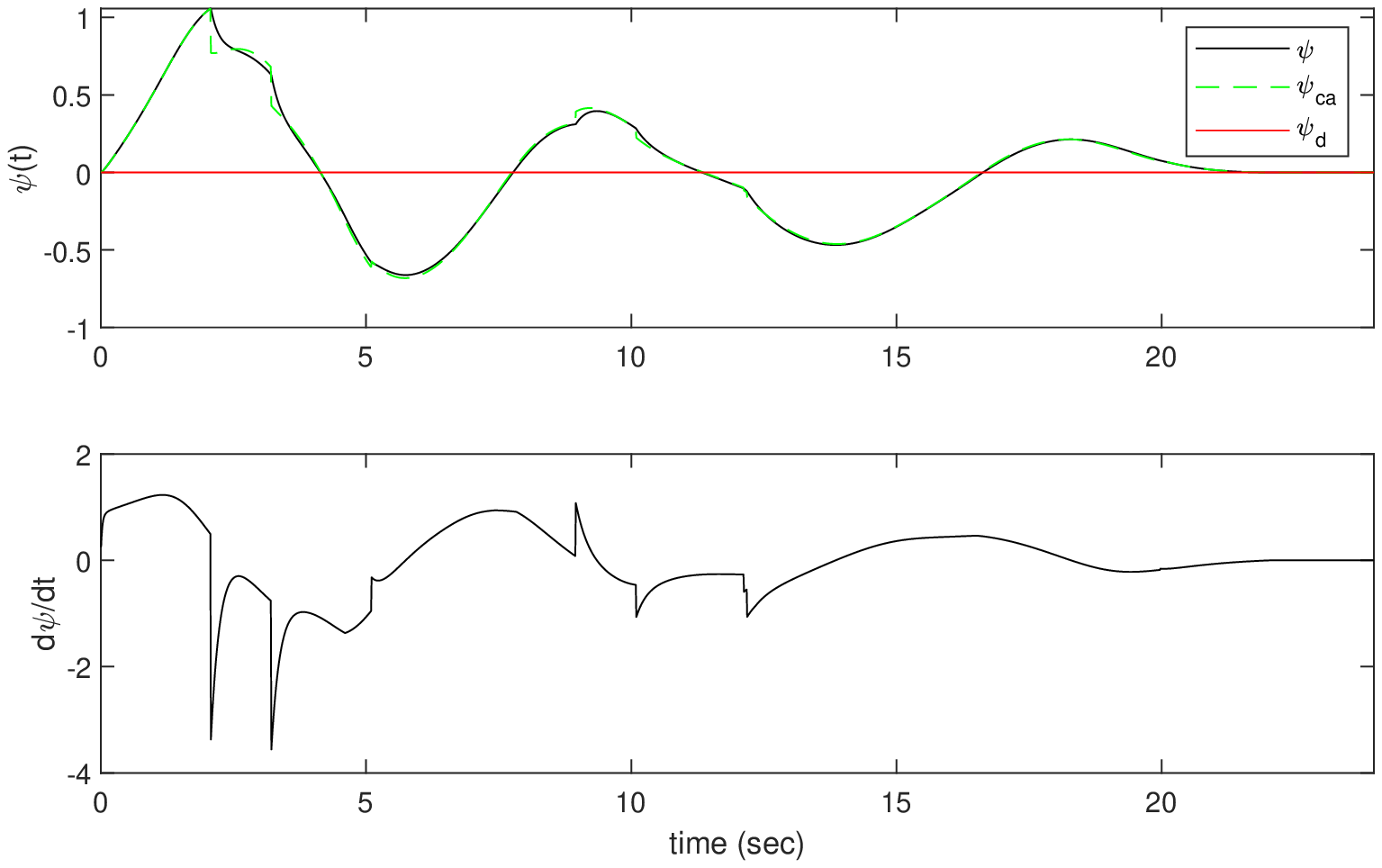}
		\caption{Steering rate controls with resulting UAV heading v/s time.}
		\label{fig::Airspace_controls}
	\end{figure}
	
	\subsection{Collision Avoidance Scenario $3$: Cluttered workspace with multiple static and moving obstacles}
	In the last simulation scenario, an UAV moving through an workspace that contains both static and moving obstacles is considered. The UAV's goal is to avoid collisions with any obstacle while maintaining an Eastward direction of motion, $\psi_d = 0$. The obstacles have different radii, $r_o^j > 0$ and may move with different speeds $V_o^j \in [0,1)$, in multiple directions $\theta_o^j \in [0,2\pi)$, $\forall j \in \mathcal{J}(t)$.
	
	As soon as the UAV registers the obstacles, it generates a mixed CAVF using Algorithm \ref{algo::1}. Then, applying the controller given in \eqref{eqn::TrackingControl}, the UAV is able to track with minimal heading error the desired vector field. The results of the simulation are illustrated in Figure \ref{fig::Clutter} for the following UAV initial conditions and parameters, respectively: $x_i = -3.3, ~y_i = 0, ~\psi_i = \psi_d = 0, ~ V = 1~\mathrm{m}/\mathrm{s},~ a = 1$ and $r_i^j = 3, ~\forall j \in \mathcal{J}(t)$. The trajectory obtained shows a more aggressive UAV behavior than in the previous simulations due to the immediate danger of colliding with the moving obstacles while intercepting static obstacles in its path. The trajectory was obtained using the control input defined in \eqref{eqn::TrackingControl} with a gain $K(t)$ that depends on the minimum distance between the UAV and the detected obstacles, as defined in the previous simulation scenario. Overall, the proposed CAVF generates a guidance field that, when tracked accurately, leads to trajectories free of any collisions with the sensed obstacles.
	\begin{figure}
		\centering
		\begin{subfigure}{0.3\textwidth}
			\centering
			\includegraphics[width=2in]{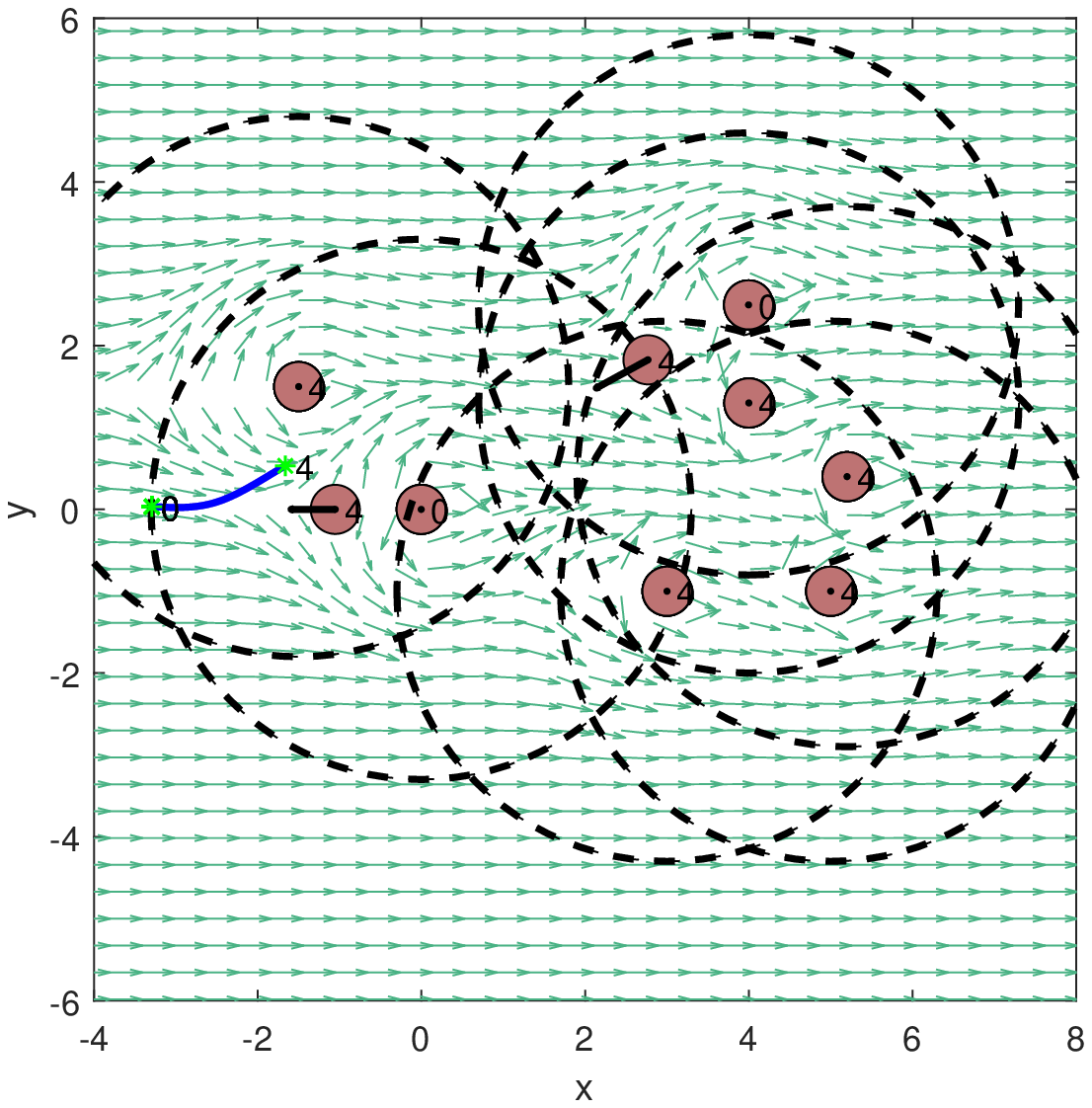}
			\caption{$t \in [0,4]$ sec}
		\end{subfigure}\hspace{1 mm}
		\begin{subfigure}{0.3\textwidth}
			\centering
			\includegraphics[width=2in]{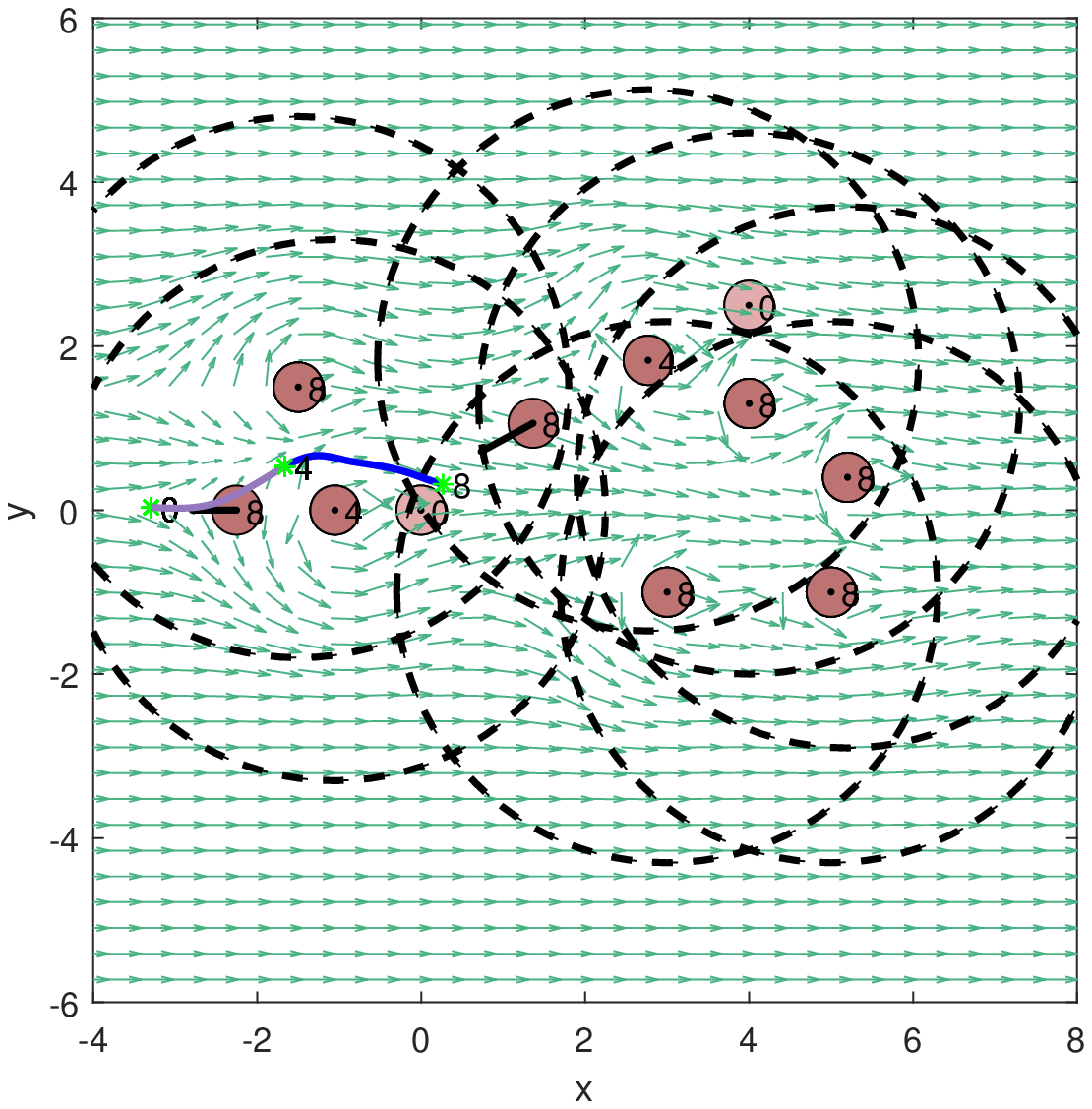}
			\caption{$t \in [4,8]$ sec}
		\end{subfigure}\hspace{1 mm}
		\begin{subfigure}{0.3\textwidth}
			\centering
			\includegraphics[width=2in]{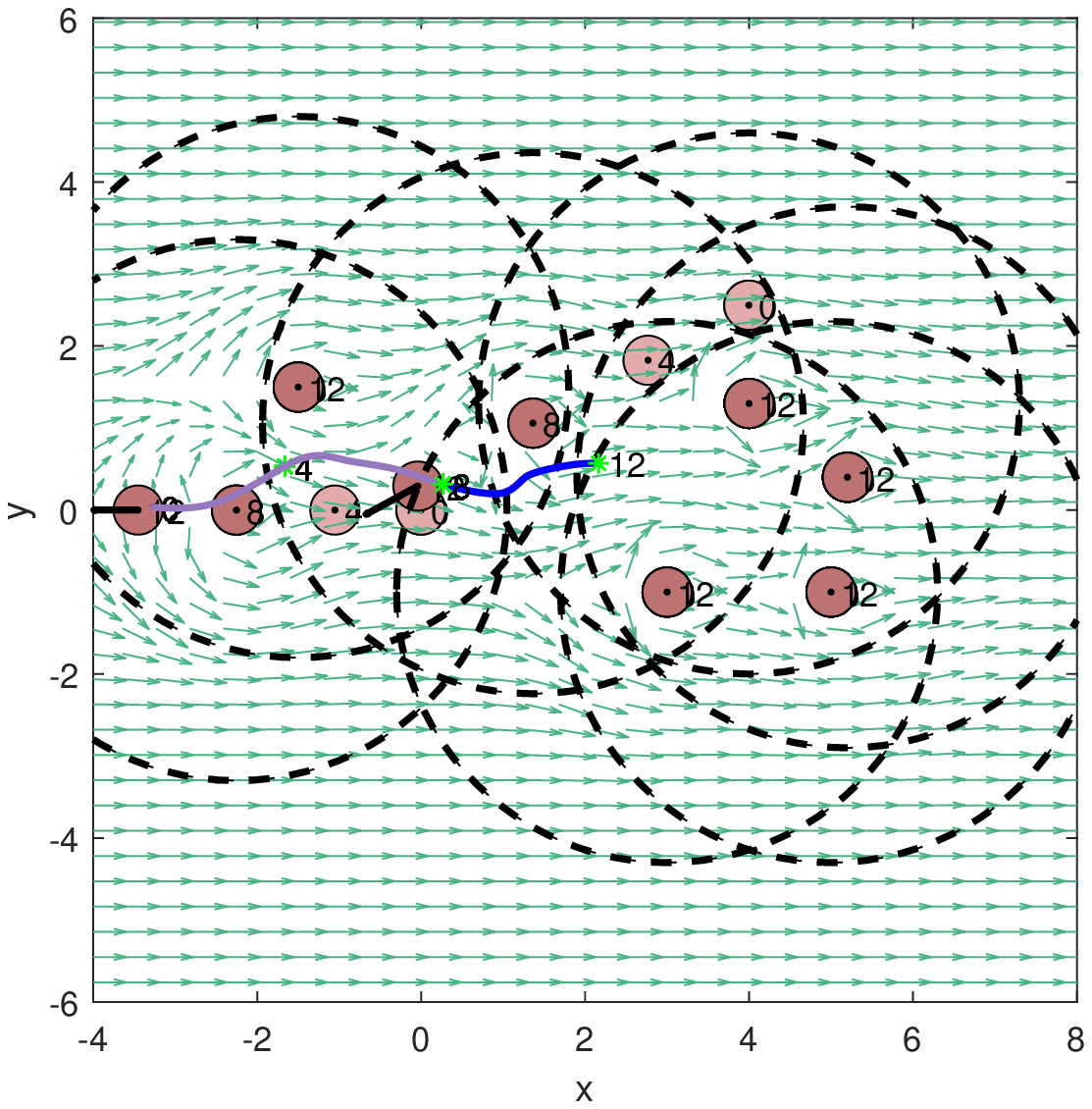}
			\caption{$t \in [8,12]$ sec}
		\end{subfigure}
		\begin{subfigure}{0.3\textwidth}
			\centering
			\includegraphics[width=2in]{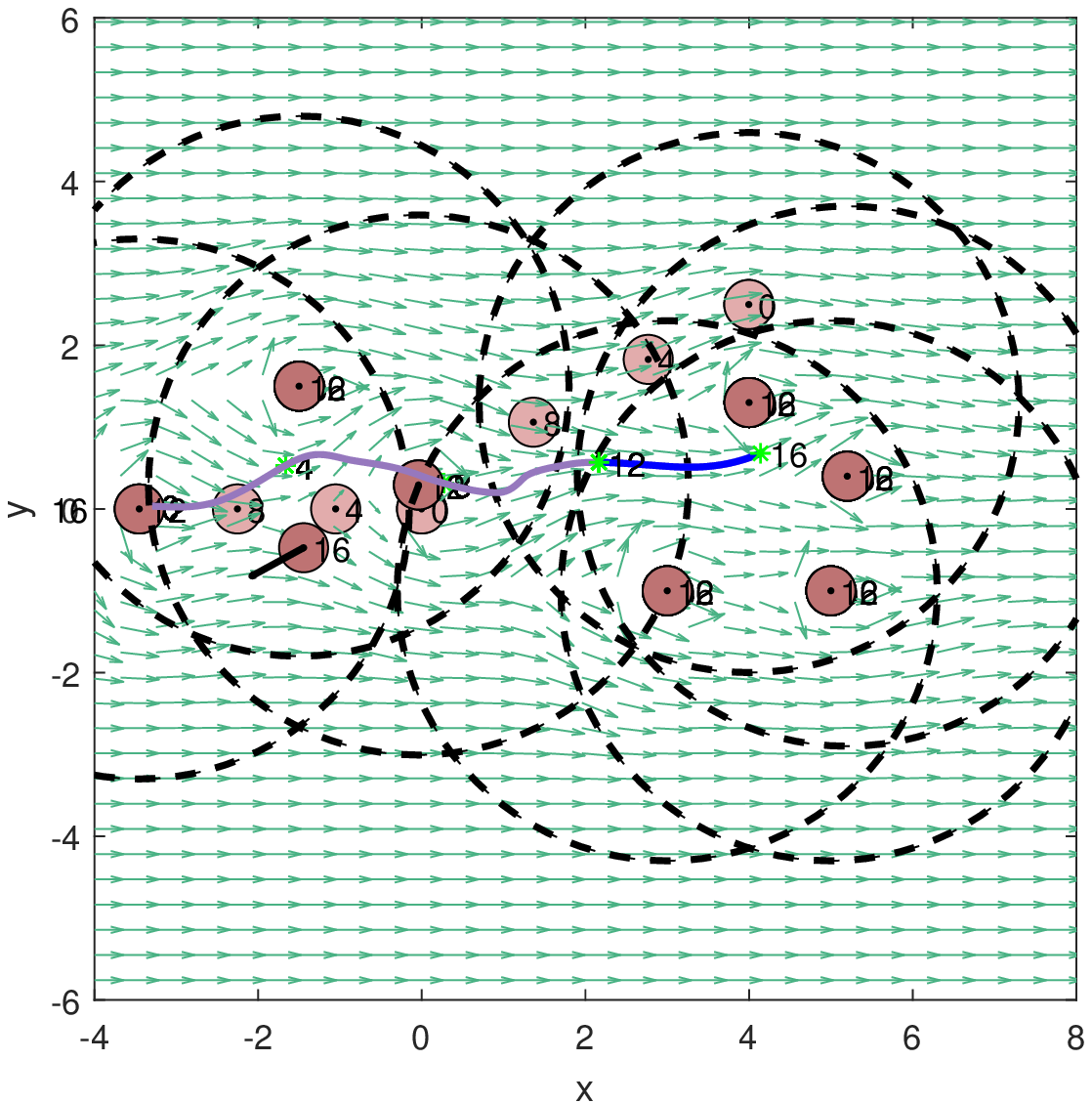}
			\caption{$t \in [12,16]$ sec}
		\end{subfigure}\hspace{1 mm}
		\begin{subfigure}{0.3\textwidth}
			\centering
			\includegraphics[width=2in]{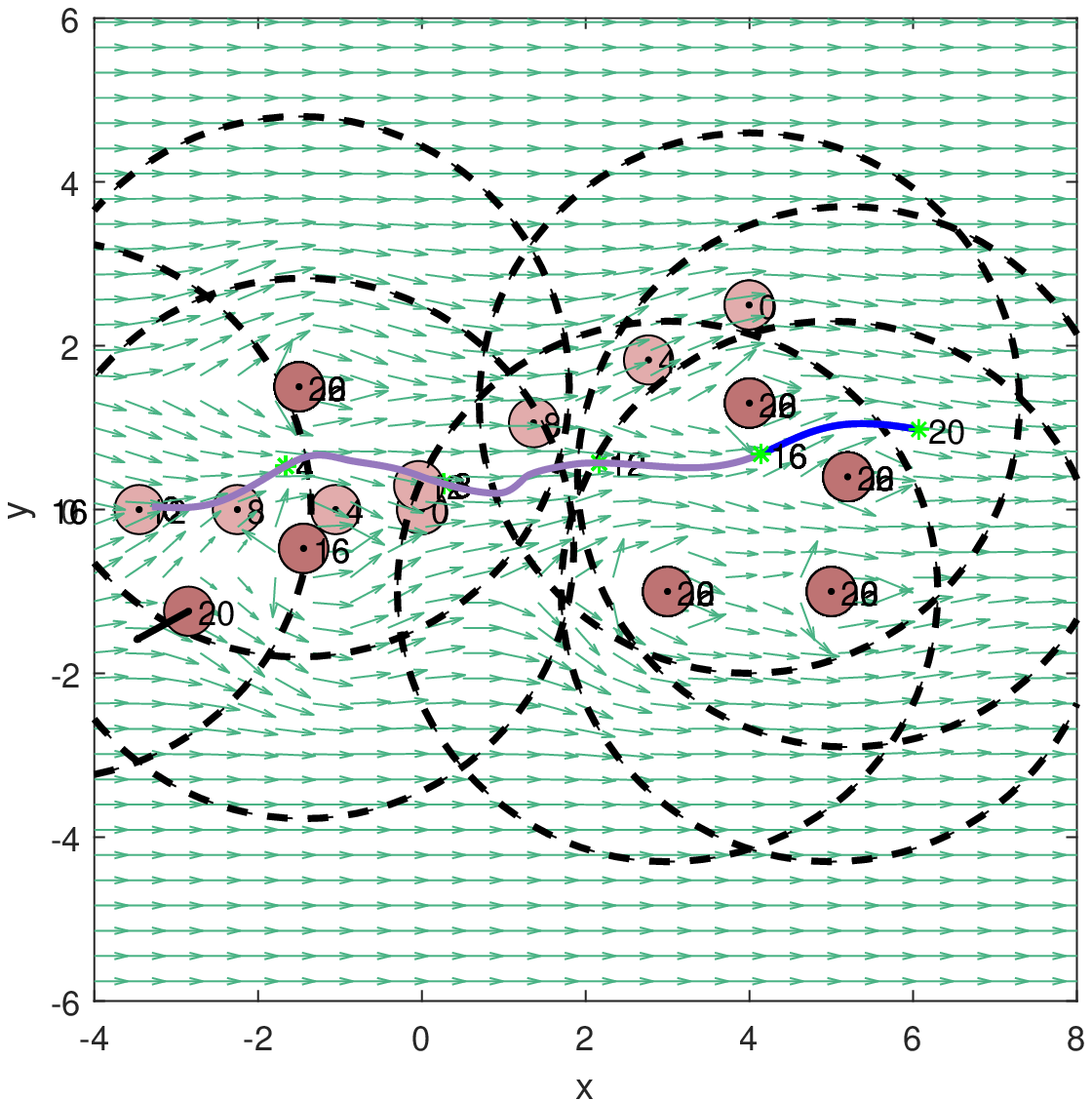}
			\caption{$t \in [16,20]$ sec}
		\end{subfigure}\hspace{1 mm}
		\begin{subfigure}{0.3\textwidth}
			\centering
			\includegraphics[width=2in]{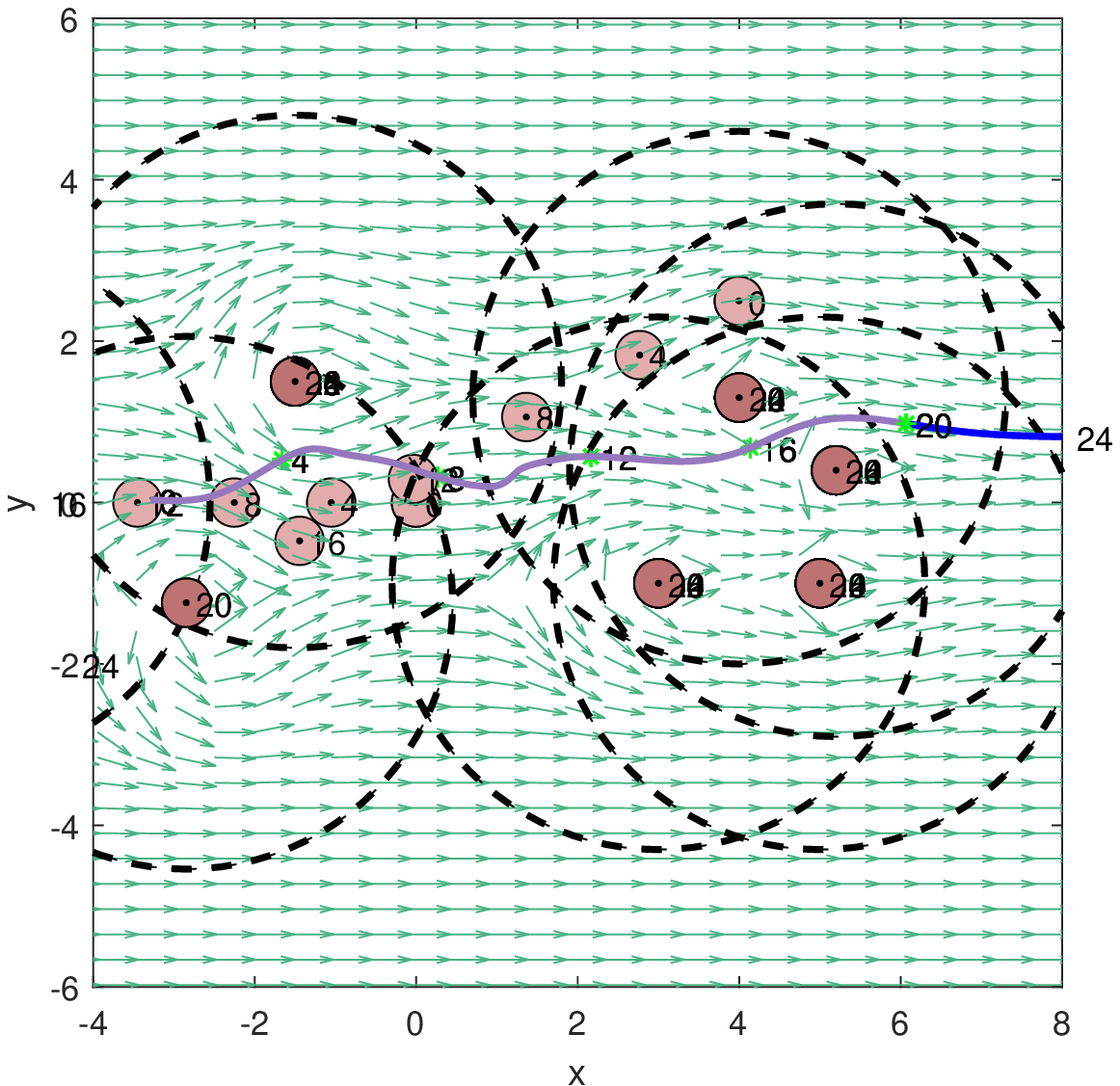}
			\caption{$t \in [20,24]$ sec}
		\end{subfigure}
		\caption{System \eqref{eqn::SysInertial} trajectories (blue) generated using \eqref{eqn::InertialControlSingleDynamicObstacle}, for a CAVF (green vector field) with the following parameters $a = 1, ~r_i = 3 ~\mathrm{m}, ~\psi_d = 0,~ V_o = 0.9 ~\mathrm{m}/\mathrm{s}, ~\theta_o = 2.35$.}
		\label{fig::Clutter}	
	\end{figure}
	\begin{figure}
		\centering
		\includegraphics[width=6.5in]{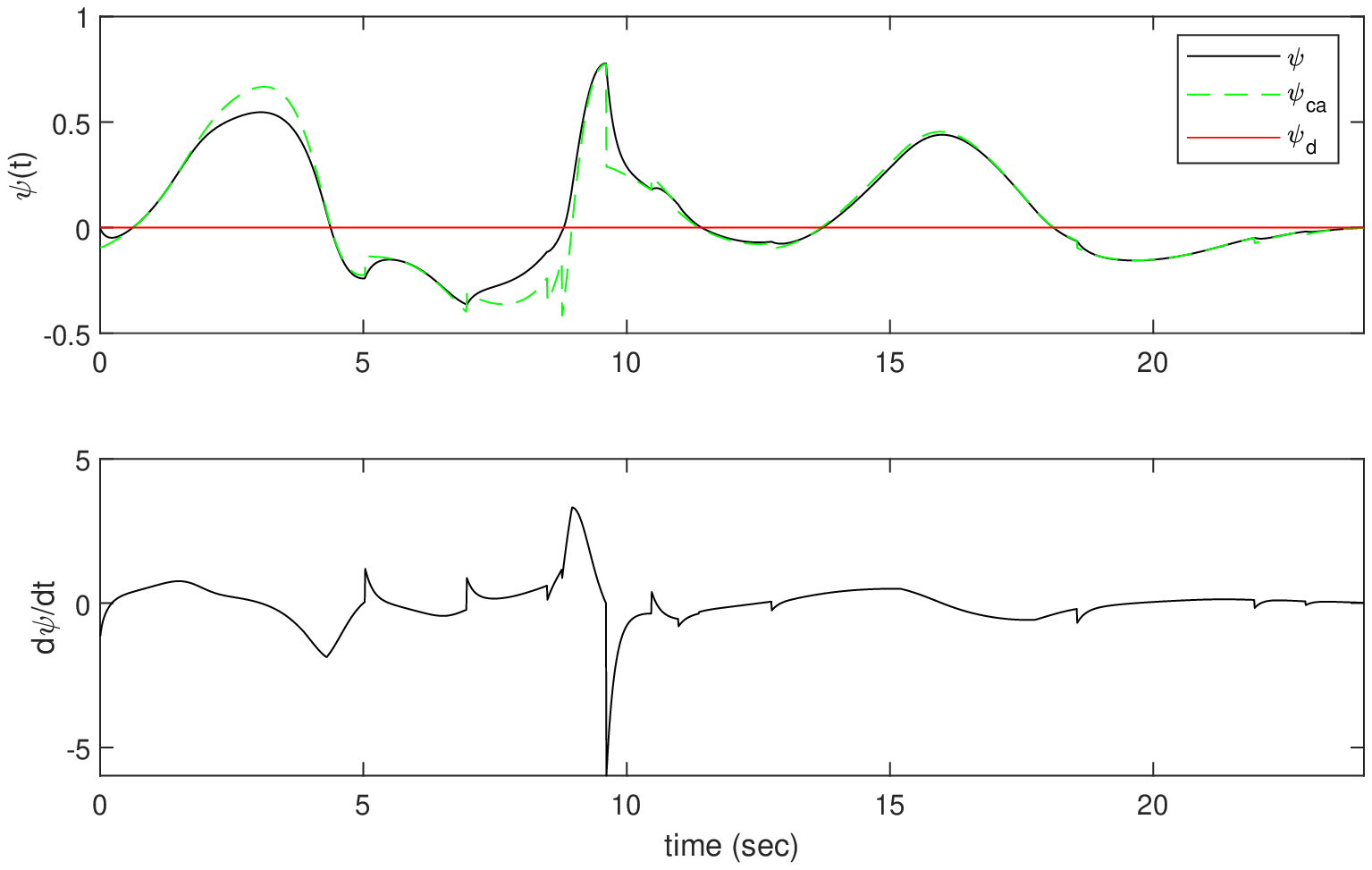}
		\caption{Steering rate controls with resulting UAV headings.}
		\label{fig::Clutter_controls}
	\end{figure}

	\section{Conclusion}
	\label{sec::concl}
	In this paper, a new methodology for collision avoidance that can be applied to UAV path planning is presented. The approach described herein makes use of a new class of guidance vector fields called collision avoidance vector fields that are determined using the relative distance between the agent and the obstacles. To perform the collision avoidance maneuvers prescribed by the vector fields, a steering law is implemented.  It is important to note that the separation between the guidance vector fields and the kinematic controllers derived in Section \ref{sec::Control}, allows the presented methodology to be used for other systems as well. For example, the collision avoidance vector field can be used as a guidance field for a simple car-model, as long as the tracking controller is changed to account for the car dynamics. 
	
	Simulations performed for three different scenarios illustrated the efficacy of the presented algorithm. In particular, it is demonstrated that the proposed approach based on collision avoidance vector fields is able to generate motion plans that take into account multiple static and moving obstacles with very little computational effort, making their generation appropriate for real-time applications. Furthermore, the proposed controllers generate trajectories which follow these motion plans accurately and within specified tolerances.  
	
	Further extensions of the collision avoidance vector field approach include but may not be limited to introducing disturbance models resulting from winds or to modifying the approach to provide collision avoidance guarantees in the presence of obstacle position and velocity uncertainty. 
	
	\section*{Appendix}
	\subsection{Proof for Proposition \ref{prop::mult_control}}
	Consider first a mixed CAVF, whose equations of motion are given by $[\dot{x}_m,\dot{y}_m]^T =  h_m(\bm{p})$, resulting from two separate CAVFs generated by isolating each obstacle with their own equations of motion: $[\dot{x}_1,\dot{y}_1]^T = h_1(\bm{p})$ and $[\dot{x}_2,\dot{y}_2]^T = h_2(\bm{p})$. The mixed CAVF is given by applying Algorithm \ref{algo::1}:
	\begin{align}
	h_m(\bm{p}) = w^1 h_1(\bm{p}) + w^2 h_2(\bm{p})
	\label{eqn::app1}
	\end{align}
	and therefore we get the following relations:
	\begin{equation}
	\begin{aligned}
	\dot{x}_m &= w^1 \dot{x}_1 + w^2 \dot{x}_2, \\
	\dot{y}_m &= w^1 \dot{y}_1 + w^2 \dot{y}_2.
	\end{aligned}
	\label{eqn::app2}
	\end{equation}
	Moreover, it is known that each vector field has the same magnitude $V$ and the following different headings: $\psi_m = \atantwo(\dot{y}_m,\dot{x}_m)$ for $h_m$, $\psi_1 = \atantwo(\dot{y}_1,\dot{x}_1)$ for $h_1$ and $\psi_2 = \atantwo(\dot{y}_2,\dot{x}_2)$ for $h_2$. Then, using the derivative of $\atantwo$, the following evolutions for each heading angle are obtained:
	\begin{equation}
	\begin{aligned}
	\dot{\psi}_1 &= (\dot{x}_1 \ddot{y}_1- \ddot{x}_1\dot{y}_1)/V^2, \\
	\dot{\psi}_2 &= (\dot{x}_2 \ddot{y}_2 - \ddot{x}_2\dot{y}_2)/V^2, \\
	\dot{\psi}_m &= \frac{(\dot{x}_m \ddot{y}_m - \ddot{x}_m\dot{y}_m)}{(w^1 V)^2 + (w^2 V)^2 + 2V^2w^1w^2\cos(\psi_1 - \psi_2)}.
	\end{aligned}
	\label{eqn::app3}
	\end{equation}
	Next, using \eqref{eqn::app2} in \eqref{eqn::app3}, the following relation is obtained for $\dot{\psi}_m$:
	\begin{equation}
		\begin{aligned}
			\dot{\psi}_m &= W^1(t) \dot{\psi}_1 + W^2(t) \dot{\psi}_2,
		\end{aligned}
	\end{equation}
    where,
    \begin{equation}
    	\begin{aligned}
    		W^1(t) &= \frac{(w^1)^2 + w^1 w^2 \cos(\psi_1 - \psi_2)}{(w^1)^2 + (w^2)^2 + 2 w^1 w^2 \cos(\psi_1 - \psi_2)}, \\
    		W^2(t) &= \frac{(w^2)^2 + w^1 w^2 \cos(\psi_1 - \psi_2)}{(w^1)^2 + (w^2)^2 + 2 w^1 w^2 \cos(\psi_1 - \psi_2)}.
    	\end{aligned}
    \end{equation}
	This implies that $\dot{\psi}_m = W^1(t) \dot{\psi}_1 + W^2(t) \dot{\psi}_2$ corresponds to the mixed CAVF for two obstacles. When there are more than two obstacles, the same approach can be carried out to get the following generalized weights: 
    \begin{equation}
	\begin{aligned}
		W^i(t) &= \frac{(w^i)^2 + \sum\limits_{\substack{j=1 \\ j\neq i}}^{n}w^i w^j \cos(\psi_i - \psi_j)}{\sum\limits_{j=1}^{n}(w^j)^2 + 2 \sum\limits_{\substack{i,j \\ i \neq j}}^{n} w^i w^j \cos(\psi_i - \psi_j)}, \text{ for all } i \in {1, 2, \dots, n}.
	\end{aligned}
	\end{equation}
	Therefore, $\dot{\psi}_m = \sum_j W^j(t) \dot{\psi}_j$. This implies that the control input \eqref{eqn::InertialControlMixedObst} corresponds to the heading evolution of a mixed CAVF. $\blacksquare$ 
	
	\subsection{Proof for Proposition \ref{prop::gain}}
	\label{app::proof}
	Let $e(t) = \psi(t) - \psi_{\text{ca}}(t)$ be the heading error between the UAV and the mixed CAVF. Then, by applying \eqref{eqn::TrackingControl} to system \eqref{eqn::SysInertial}, the following error dynamics are obtained:  $\dot{e}(t) = -Ke(t) + u_m(t) - \dot{\psi}_{\text{ca}}(t)$. It is straightforward to see that $\dot{\psi}_{\text{ca}} = u_m(t)$, which results in $\dot{e}(t) = -Ke(t)$. Therefore, $e(t) = c\exp(-K(t-t_i))$, where $c$ is a constant that depends on the initial heading error, $c = \psi(t_i) - \psi_{\text{ca}}(t_i)$. 
		
	Next, suppose that the heading error satisfies $e(t) \leq e_\psi, ~\forall t > t_\text{track} > t_i$, where $t_\text{track}$ will be defined later. Equivalently, $c \exp(-K(t-t_i)) \leq e_\psi$, or, expressed differently, $K(t-t_i) \geq \log{c} - \log{e_\psi}$. Therefore, if $K$ satisfies the inequality
	\begin{align}
	K \geq \frac{\log{c}-\log{e_\psi}}{t_\text{track}-t_i},
	\label{eqn::GainFormula}
	\end{align}
	then $e(t_\text{track}) \leq e_\psi$ which directly implies that the heading error converges with the given tolerance, that is, $e(t)\leq e_\psi ~ \forall t \geq t_{\text{track}}$. Moreover, since the agent travels with constant speed, time $t_\text{track}$ can be selected such that $t_\text{track} = \delta/(2V)$, where $\delta$ is the minimum separation distance between obstacles, assumed for mixed CAVF. Then, selecting the gain to be:
	\begin{align}
	\hat{K} = \frac{\log{c}-\log{e_\psi}}{\frac{\delta}{2V}-t_i},
	\label{eqn::Gain}
	\end{align}
	guarantees convergence of the agent's heading to the mixed CAVF within a ball of radius $\delta/2$. Note that gain $\hat{K}$ depends on the initial error between the agent heading and the CAVF which is bounded from above by $\pi$. Therefore, gain \eqref{eqn::TrackingGain} that depends only on the separation distance and error tolerance can be used for any situation instead of \eqref{eqn::Gain}. $\blacksquare$
	
	\section*{Acknowledgments}
	This work was supported in part by the National Science Foundation award 1562339. The authors would like to thank student Riccardo Caccavale of Sapienza Universita di Roma for bringing to our attention an error in our previous proof for Proposition \ref{prop::mult_control} and a typographical error in our previous presentation of Equation \ref{eqn::cavf_moving_rotation}.
	
	\bibliography{SteeringControl2}

\begin{thebibliography}{26}
\newcommand{\enquote}[1]{``#1''}
\providecommand{\natexlab}[1]{#1}
\providecommand{\url}[1]{\texttt{#1}}
\providecommand{\urlprefix}{URL }
\expandafter\ifx\csname urlstyle\endcsname\relax
  \providecommand{\doi}[1]{doi:\discretionary{}{}{}#1}\else
  \providecommand{\doi}{doi:\discretionary{}{}{}\begingroup
  \urlstyle{rm}\Url}\fi

\bibitem[{Frew et~al.(2008)Frew, Lawrence, and Morris}]{frew2008coordinated}
Frew, E.~W., Lawrence, D.~A., and Morris, S., \enquote{Coordinated standoff
  tracking of moving targets using Lyapunov guidance vector fields,}
  \emph{Journal of guidance, control, and dynamics}, Vol.~31, No.~2, 2008, pp.
  290--306.

\bibitem[{Dijkstra(1959)}]{dijkstra}
Dijkstra, E.~W., \enquote{A note on two problems in connexion with graphs,}
  \emph{Numerische mathematik}, Vol.~1, No.~1, 1959, pp. 269--271.

\bibitem[{Hart et~al.(1968)Hart, Nilsson, and Raphael}]{A*}
Hart, P.~E., Nilsson, N.~J., and Raphael, B., \enquote{A formal basis for the
  heuristic determination of minimum cost paths,} \emph{IEEE transactions on
  Systems Science and Cybernetics}, Vol.~4, No.~2, 1968, pp. 100--107.

\bibitem[{Yang and Zhao(2004)}]{yangZhao}
Yang, H., and Zhao, Y., \enquote{Trajectory planning for autonomous aerospace
  vehicles amid known obstacles and conflicts,} \emph{Journal of Guidance,
  Control, and Dynamics}, Vol.~27, No.~6, 2004, pp. 997--1008.

\bibitem[{Stentz(1994)}]{D*}
Stentz, A., \enquote{Optimal and efficient path planning for partially-known
  environments,} \emph{Proceedings IEEE International Conference on Robotics
  and Automation}, 1994.

\bibitem[{Zhu and Latombe(1991)}]{zhuLatombe}
Zhu, D., and Latombe, J.-C., \enquote{New heuristic algorithms for efficient
  hierarchical path planning,} \emph{IEEE Transactions on Robotics and
  Automation}, Vol.~7, No.~1, 1991, pp. 9--20.

\bibitem[{Huttenlocher et~al.(1993)Huttenlocher, Kedem, and Sharir}]{sharir}
Huttenlocher, D.~P., Kedem, K., and Sharir, M., \enquote{The upper envelope of
  Voronoi surfaces and its applications,} \emph{Discrete \& Computational
  Geometry}, Vol.~9, No.~3, 1993, pp. 267--291.

\bibitem[{Schwartz and Sharir(1988)}]{sharir2}
Schwartz, J.~T., and Sharir, M., \enquote{A survey of motion planning and
  related geometric algorithms,} \emph{Artificial Intelligence}, Vol.~37, No.
  1-3, 1988, pp. 157--169.

\bibitem[{{\'O}'D{\'u}nlaing and Yap(1985)}]{retract}
{\'O}'D{\'u}nlaing, C., and Yap, C.~K., \enquote{A “retraction” method for
  planning the motion of a disc,} \emph{Journal of Algorithms}, Vol.~6, No.~1,
  1985, pp. 104--111.

\bibitem[{Kuffner and LaValle(2000)}]{lavalle}
Kuffner, J.~J., and LaValle, S.~M., \enquote{RRT-connect: An efficient approach
  to single-query path planning,} \emph{Robotics and Automation, 2000.
  Proceedings. ICRA'00. IEEE International Conference on}, IEEE, 2000, pp.
  995--1001.

\bibitem[{Karaman and Frazzoli(2010)}]{rrt*}
Karaman, S., and Frazzoli, E., \enquote{Incremental sampling-based algorithms
  for optimal motion planning,} \emph{Robotics Science and Systems VI}, Vol.
  104, 2010, p.~2.

\bibitem[{Frazzoli et~al.(2002)Frazzoli, Dahleh, and Feron}]{frazzoli-agile}
Frazzoli, E., Dahleh, M.~A., and Feron, E., \enquote{Real-time motion planning
  for agile autonomous vehicles,} \emph{Journal of Guidance, Control, and
  Dynamics}, Vol.~25, No.~1, 2002, pp. 116--129.

\bibitem[{Tedrake et~al.(2010)Tedrake, Manchester, Tobenkin, and
  Roberts}]{lqrtrees}
Tedrake, R., Manchester, I.~R., Tobenkin, M., and Roberts, J.~W.,
  \enquote{LQR-trees: Feedback motion planning via sums-of-squares
  verification,} \emph{The International Journal of Robotics Research},
  Vol.~29, No.~8, 2010, pp. 1038--1052.

\bibitem[{Upadhyay and Ratnoo(2017)}]{upda-ratnoo}
Upadhyay, S., and Ratnoo, A., \enquote{Smooth Path Planning for Unmanned Aerial
  Vehicles with Airspace Restrictions,} \emph{Journal of Guidance, Control, and
  Dynamics}, Vol.~40, No.~7, 2017, pp. 1596--1612.

\bibitem[{Mattei and Blasi(2010)}]{mattei}
Mattei, M., and Blasi, L., \enquote{Smooth flight trajectory planning in the
  presence of no-fly zones and obstacles,} \emph{Journal of guidance, control,
  and dynamics}, Vol.~33, No.~2, 2010, pp. 454--462.

\bibitem[{Delingette et~al.(1991)Delingette, Hebert, and Ikeuchi}]{ikeuchi}
Delingette, H., Hebert, M., and Ikeuchi, K., \enquote{Trajectory generation
  with curvature constraint based on energy minimization,} \emph{Proceedings
  IEEE/RSJ International Conference Intelligent Robots and Systems}, IEEE,
  1991, pp. 206--211.

\bibitem[{Sun et~al.(2017)Sun, Liu, Dai, and Grymin}]{sunliudai}
Sun, C., Liu, Y.-C., Dai, R., and Grymin, D., \enquote{Two Approaches for Path
  Planning of Unmanned Aerial Vehicles with Avoidance Zones,} \emph{Journal of
  Guidance, Control, and Dynamics}, Vol.~40, No.~8, 2017, pp. 2076--2083.

\bibitem[{Frazzoli et~al.(2001)Frazzoli, Mao, Oh, and Feron}]{frazzoliSDP}
Frazzoli, E., Mao, Z.-H., Oh, J.-H., and Feron, E., \enquote{Resolution of
  conflicts involving many aircraft via semidefinite programming,}
  \emph{Journal of Guidance, Control, and Dynamics}, Vol.~24, No.~1, 2001, pp.
  79--86.

\bibitem[{Khatib(1986)}]{khatib}
Khatib, O., \enquote{Real-time obstacle avoidance for manipulators and mobile
  robots,} \emph{Autonomous robot vehicles}, Springer, 1986, pp. 396--404.

\bibitem[{Waydo and Murray(2003)}]{waydo}
Waydo, S., and Murray, R.~M., \enquote{Vehicle motion planning using stream
  functions,} \emph{Robotics and Automation, 2003. Proceedings. ICRA'03. IEEE
  International Conference on}, Vol.~2, IEEE, 2003, pp. 2484--2491.

\bibitem[{Li and Bui(1998)}]{libui}
Li, Z., and Bui, T., \enquote{Robot path planning using fluid model,}
  \emph{Journal of Intelligent and Robotic Systems}, Vol.~21, No.~1, 1998, pp.
  29--50.

\bibitem[{Yao et~al.(2015)Yao, Wang, and Su}]{zikang}
Yao, P., Wang, H., and Su, Z., \enquote{UAV feasible path planning based on
  disturbed fluid and trajectory propagation,} \emph{Chinese Journal of
  Aeronautics}, Vol.~28, No.~4, 2015, pp. 1163--1177.

\bibitem[{Khansari-Zadeh and Billard(2012)}]{zadehDSA}
Khansari-Zadeh, S.~M., and Billard, A., \enquote{A dynamical system approach to
  realtime obstacle avoidance,} \emph{Autonomous Robots}, Vol.~32, No.~4, 2012,
  pp. 433--454.

\bibitem[{Lau et~al.(2015)Lau, Eden, and Oetomo}]{dLau}
Lau, D., Eden, J., and Oetomo, D., \enquote{Fluid motion planner for
  nonholonomic 3-D mobile robots with kinematic constraints,} \emph{IEEE
  Transactions on Robotics}, Vol.~31, No.~6, 2015, pp. 1537--1547.

\bibitem[{Owen et~al.(2011)Owen, Hillier, and Lau}]{owen}
Owen, T., Hillier, R., and Lau, D., \enquote{Smooth path planning around
  elliptical obstacles using potential flow for non-holonomic robots,}
  \emph{Robot Soccer World Cup}, Springer, 2011, pp. 329--340.

\bibitem[{Lawrence et~al.(2008)Lawrence, Frew, and
  Pisano}]{lawrence2008lyapunov}
Lawrence, D.~A., Frew, E.~W., and Pisano, W.~J., \enquote{Lyapunov vector
  fields for autonomous unmanned aircraft flight control,} \emph{Journal of
  Guidance, Control, and Dynamics}, Vol.~31, No.~5, 2008, pp. 1220--1229.

\end{thebibliography}
	
\end{document}